%% file: main.tex
\documentclass[letterpaper]{article} 
\usepackage[preprint]{aaai2027} 
\usepackage[hyphens]{url} 
\usepackage{graphicx} 
\usepackage{natbib} 
\usepackage{caption} 
\usepackage{amsmath}
\usepackage{amssymb}
\usepackage{amsthm}
\usepackage{algorithm}
\usepackage{algorithmic}
\usepackage{booktabs}
\usepackage{array}
\usepackage{xcolor}
\usepackage{microtype}
\newcommand{\model}{\textsc{ALPHABET}}
\newcommand{\R}{\mathbb{R}}
\newcommand{\C}{\mathbb{C}}
\newcommand{\E}{\mathbb{E}}

\newcommand{\BenchmarkTaskCount}{82}

\newcommand{\BenchmarkCompleteTaskCount}{82}

\newcommand{\BenchmarkUCRTaskCount}{30}
\newcommand{\BenchmarkGeneralTaskCount}{21}
\newcommand{\BenchmarkUEAFaultTaskCount}{31}

\newcolumntype{L}[1]{>{\raggedright\arraybackslash}p{#1}}
\newtheoremstyle{pacstatement}{6pt}{6pt}{\normalfont}{}
  {\bfseries}{.}{0.5em}{}
\theoremstyle{pacstatement}
\newtheorem{theorem}{Theorem}
\newtheorem{proposition}{Proposition}
\newtheorem{corollary}{Corollary}
\newtheorem{lemma}{Lemma}

\title{ALPHABET: A Laplace-Pole History Aggregator \\ with Banked Exponential Transport}
\author{Daehwa Ko, JaeHyeon Kim, Oh Seong Kwon, and Jay Hoon Jung}
\affiliations{Department of Artificial Intelligence, Korea Aerospace University, Goyang, Republic of Korea\\
daehwa001210@gmail.com, kjh990127@kau.kr, kos\_25@kau.kr, jhjung@kau.ac.kr}

\begin{document}
\maketitle

\begin{abstract}


  
Can a sequence model remain competitive with only a few thousand parameters
and an explicitly auditable prediction interface? We introduce \model{}, a
compact linear-time model that compresses temporal history into stable
complex pole modes: a direct bank synthesizes its modal states back into the
feature trajectory, an independent cascaded bank analyzes the transformed
trajectory without resynthesis, and an affine head reads only modal energies
and lag moments from both banks.
We characterize the temporal information this descriptor retains: for a
stationary, fully observed feature process, each mode energy is a
frequency-localized measurement of the second-order spectrum, the continuum
of such measurements identifies the spectrum, and almost every mode separates
any fixed finite set of spectrally distinct classes. On a Gaussian control
with matched low-lag statistics, the learned descriptor approaches the Bayes
oracle where raw autocovariances remain at chance.
  
Across the fixed \BenchmarkTaskCount-task registry, \model{} attains mean
rank $3.97$ in the complete ten-family comparison. At the common-width
\(D=64\) runtime anchor, its 6,437 parameters deliver \(5.02\times\) faster
inference and \(3.93\times\) faster complete training steps than the nine
baselines on average.
\end{abstract}

\section{Introduction}
\label{sec:introduction}

A sequence model turns temporal history into a state that a predictor can use.
Modern architectures have made that state increasingly powerful: gated recurrence propagates hidden states token by token, attention constructs
contextual representations from pairwise interactions, and structured or
selective state-space models evolve states through learned dynamics
\cite{hochreiter1997lstm,cho2014gru,vaswani2017attention,gu2022s4,gu2023mamba}.
Yet competitive temporal prediction should not require a large, slow, and
opaque backbone---and opacity has a concrete cost: it is rarely stated which
temporal information the retained state preserves. Without such an account,
the adequacy of a compact state is assessed mainly through downstream
accuracy, and it remains unclear where its guarantees end.

We therefore study how much discriminative temporal structure a deliberately
compact, stable state can preserve. Classical model reduction asks how few dynamical degrees
of freedom approximate a known input--output system
\cite{moore1981pca,antoulas2005approximation,benner2015modelreduction}.
We pose the analogous question for prediction: a sequence encoder summarizes its
recurrent states into a compact prediction interface, aiming not to
reconstruct every input token but to expose a small set of temporal
measurements whose meaning, stability, and empirical utility can all be
checked.

\begin{figure*}[t!]
\centering
\includegraphics[width=.98\textwidth]{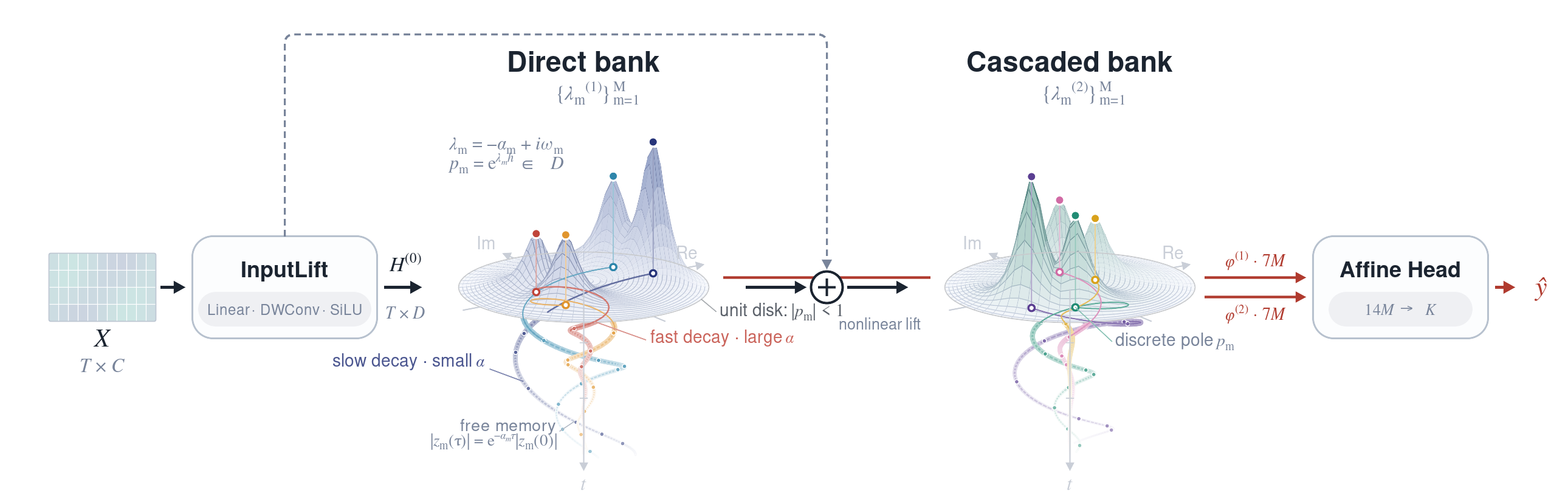}
\caption{\model{} architecture. The direct bank produces a $7{\times}M$
modal descriptor (seven coordinates for each of $M$ modes) and synthesizes
$H^{(1)}$; the independent cascaded bank returns a second $7{\times}M$
descriptor without resynthesis. Their concatenation forms the
$14{\times}M$-dimensional affine prediction interface.}
\label{fig:architecture}
\end{figure*}

\model{} uses two banks of complex recurrent modes parameterized by stable
Laplace poles, with each pole specifying the decay and oscillation of its
mode. Pole locations may be trained, but the spectral coverage of the bank
and its modal moment readout are the essential construction. An input lift
first forms a real-valued feature trajectory. The direct pole bank then
accumulates projected features through interval-aware modal dynamics,
projects the resulting complex states back into the real feature space, and
adds them to the feature trajectory. An independent cascaded pole bank
applies a second set of modal dynamics to the resulting, nonlinearly
transformed trajectory---a feature process shaped by direct-bank synthesis
rather than the original lifted trajectory---without resynthesis. Both banks
return mode-wise energies and complex-lag moments, whose concatenation forms a fixed affine prediction interface.
  
The representation also has a precise population interpretation: for a
stationary feature process, each pole energy is a directional,
Poisson-localized measurement of its matrix-valued spectrum.
The continuum of admissible measurements uniquely identifies that spectrum, and almost
every admissible mode separates any fixed finite collection of distinct class spectra.
Whenever a finite bank contains such a mode, there exists an affine head that
exactly separates the resulting population prototypes. Under ergodicity, the
corresponding empirical decision becomes consistent as sequence length grows.
Independently, stable normalized dynamics bound the realized descriptor
uniformly in sequence length. This analysis characterizes the information
available to the model for a fixed feature process. Because the head is
affine in these modal statistics, auditability is concrete: each predicted
margin splits exactly into per-pole terms, and removing the top-attributed
poles degrades accuracy far beyond random removal.
  
The empirical study asks whether this structured compression retains useful
predictive information in practice. Our contributions are threefold. First, we introduce a linear-time sequence model with a direct pole bank, an
independent cascaded pole bank, and a fixed affine prediction interface that
reads only mode-wise energies and lag moments from both banks. Second, for a fixed, fully observed, equal-step,
second-order stationary feature process, we characterize mode energies as
frequency-localized spectral measurements, show that the continuum family
identifies the second-order spectrum, and establish generic population
separation for finite collections of spectrally distinct classes. Separately,
stable interval-aware dynamics and normalized excitation yield
sequence-length-independent bounds for realized finite-bank prediction
interfaces. Third, we test whether end-to-end learning realizes useful
finite-mode representations through a fixed \BenchmarkTaskCount-task registry
and a complete \BenchmarkCompleteTaskCount-task comparison against nine
trainable sequence-model families, a moment-matched spectral diagnostic,
and boundary and systems audits across data regimes, corruptions, sequence
shapes, and execution phases.

\section{Related Work}
\label{sec:related}
\paragraph{Temporal state construction.}
Recurrent models such as LSTMs and GRUs propagate gated hidden states through
time, while Transformers construct contextual representations through
attention 
\cite{hochreiter1997lstm,cho2014gru, vaswani2017attention}.
Structured and diagonal state-space models, including HiPPO, S4, DSS,
S4D, S5, LRUs, and Mamba, instead represent temporal history through learned linear or selective dynamics 
\cite{gu2020hippo,gu2022s4,gupta2022dss,gu2022s4d,smith2023s5, orvieto2023lru,gu2023mamba}.
Models for irregular observations, such as GRU-D, ODE-RNN, Latent ODE, and Neural CDE, additionally incorporate elapsed
time through explicit decay, continuous-time evolution, or controlled differential equations 
\cite{che2018grud,rubanova2019latentode,kidger2020neuralcde}.
\model{} shares the linear-time recurrent form of diagonal state-space models
but differs in its prediction interface: the head never sees the state
trajectory itself, only mode-wise energies and complex-lag moments from the
two banks.
  
\paragraph{Structured spectral representations.}
Balanced reduction, dynamic mode decomposition, and Laguerre or Kautz bases
seek compact dynamical descriptions
\cite{moore1981pca,antoulas2005approximation,schmid2010dmd,
wahlberg1996laguerre,ninness1997orthonormal}. Filter-output covariance and
state covariance are classical spectral interpolation objects
\cite{byrnes2000three,georgiou2002statecovariance}; wavelet variances, random
autoregressive filter banks, recurrent covariance features, and complex-pole
arrays provide related localized summaries
\cite{li2002waveletspectrum,farahmand2017rpfb,gilson2020covariance,
pintoorellana2021cofre}. Dedicated time-series models---random convolutional
pipelines such as ROCKET and MiniRocket
\cite{dempster2020rocket,dempster2021minirocket} and specialized
architectures such as InceptionTime and TimesNet
\cite{fawaz2020inceptiontime,wu2023timesnet}---achieve strong accuracy on
the same archives; our comparison instead spans generic trainable sequence
backbones under one shared protocol covering classification, forecasting,
and irregular sampling, so we view them as complementary rather than
competing baselines. Our use of learned poles is supervised, but the main
distinction is not the filter family alone: we connect the implemented modal
statistics to a characteristic continuum embedding, generic affine
separation, and a bounded prediction interface.
  
\section{ALPHABET}
\paragraph{Notation and shapes.}
Suppressing the batch dimension, let
\(X=[x_1,\ldots,x_T]^\top\in\R^{T\times C}\),
\(h\in\R_{\geq0}^{T}\), \(o\in\{0,1\}^{T\times C}\), and
\(w\in\{0,1\}^{T}\). Here, \(h_t\) is the elapsed time from step
\(t-1\) to \(t\), while \(o_{t,c}\) and \(w_t\) indicate channel
availability and token validity, respectively; a token-level scalar
mask is broadcast across channels. We define the event-level
availability
\(\widetilde o_t=\max_c o_{t,c}\).
The mask \(o\) removes unavailable raw entries, its reduction
\(\widetilde o\) gates external excitation, and \(w\) excludes padded
or invalid tokens from local features and moment estimates.

Let \(D\), \(M\), and \(K\) denote the real feature width, the number
of complex modes per bank, and the output dimension, respectively,
and let \(r=1,2\) index the direct and cascaded banks, respectively.
\[
\begin{gathered}
H^{(r-1)},U^{(r)}\in\R^{T\times D},\quad
Z^{(r)}\in\C^{T\times M},\quad W_X\in\R^{D\times C},\\
A_r\in\R^{D\times2M},\quad
A_r^\top A_r=I_{2M},\quad 2M\leq D.
\end{gathered}
\]
The paired coordinates of \(A_r\) form the real and imaginary parts
of the \(M\) complex excitations. The finite set
\(\mathcal T\subset\mathbb N_{>0}\) contains the token lags used by the
modal readout, and \(s_1,\gamma_1\in\R^D\) parameterize the direct-bank
residual and feedthrough terms. We write
\(\phi^{(r)}(X)\), \(g(X)\), and \(f_\Theta(X)\), suppressing their
dependence on \(h,o,w\).

\subsection{Stable Laplace-pole dynamics}
Fix a bank \(r\in\{1,2\}\) and suppress its index. For each mode
\(m\in\{1,\ldots,M\}\), let
\[
\lambda_m=-\alpha_m+i\omega_m,\qquad \alpha_m>0,
\]
where \(\alpha_m\) and \(\omega_m\) are the decay rate and angular
frequency. \model{} maps \(v_t\in\R^D\) to complex excitations
\(e_t\in\C^M\), with \(m\)-th component \(e_{m,t}\).

Let \(p_{m,t}=e^{\lambda_mh_t}\) and \(z_{m,t}\in\C\) denote the
discrete pole and modal state. Under a zero-order hold on \(e_{m,t}\),
the exact evolution is
\begin{equation}
z_{m,t}
=
p_{m,t}z_{m,t-1}
+
\chi_t^{(r)}\frac{p_{m,t}-1}{\lambda_m}e_{m,t}.
\label{eq:recurrence}
\end{equation}
The injection gates are
\[
\chi_t^{(1)}=w_t\widetilde o_t,
\qquad
\chi_t^{(2)}=w_t.
\]
Thus \(w_t\) blocks injection at invalid tokens, while
\(\widetilde o_t\) additionally blocks direct-bank injection at fully
unobserved tokens; neither stops the elapsed-time state transition.
The cascaded bank omits \(o\) to preserve memory transported by the
direct bank.

For \(h_t>0\),
\[
|p_{m,t}|=e^{-\alpha_mh_t}<1.
\]
Regular sampling therefore yields a fixed pole \(p_m\), whereas
irregular sampling changes \(p_{m,t}\) with \(h_t\) while preserving
stability.

\subsection{Modal moment readout}
For either bank, suppress the bank index and set
\(w_{t,\tau}=w_tw_{t-\tau}\) for in-range pairs and zero otherwise.
For each \(\tau\in\{0\}\cup\mathcal T\), mode \(m\) contributes
\begin{equation}
R_{m,\tau}
=
\frac{
\sum_t w_{t,\tau}
z_{m,t}\overline{z_{m,t-\tau}}
}{
\max\{1,\sum_t w_{t,\tau}\}
}.
\label{eq:modal-moments}
\end{equation}
Here \(R_{m,0}\) is the empirical modal energy, while the complex nonzero-lag
moments retain phase-sensitive temporal dependence.
Define
\[
\operatorname{rlog}(u)
=
\begin{cases}
\log(1+|u|)\,u/|u|, & u\neq0,\\
0, & u=0.
\end{cases}
\]
Representing each complex value by its real and imaginary parts, the
descriptor of mode \(m\) is
\begin{equation}
\phi_m
=
\Bigl[
\log(1+R_{m,0})
\ \Big|\
\bigl(\operatorname{rlog}R_{m,\tau}\bigr)_{\tau\in\mathcal T}
\Bigr]
\in\R^{1+2|\mathcal T|}.
\label{eq:lag-readout}
\end{equation}
We use \(\mathcal T=\{1,2,4\}\), giving seven real coordinates per mode
and \(7M\) coordinates per bank. Let
\(\phi^{(r)}(X)\in\R^{7M}\) concatenate the mode descriptors from bank
\(r\). The final descriptor and logits are
\(g(X)=[\phi^{(1)}(X)\mid\phi^{(2)}(X)]\in\R^{14M}\) and
\(f_\Theta(X)=W_{\rm head}g(X)+b_{\rm head}\in\R^K\).

\begin{algorithm}[t]
\caption{\model{} forward pass}
\label{alg:alphabet}
\small
\begin{algorithmic}[1]
\REQUIRE Values $X$, intervals $h$, observation mask $o$, validity mask $w$,
lag set $\mathcal T$
\STATE $\alpha_m^{(r)}\leftarrow\alpha_{\min}
+(\alpha_{\max}-\alpha_{\min})\operatorname{sigmoid}(\delta_m^{(r)})$
\STATE $\omega_m^{(r)}\leftarrow\omega_{\max}\tanh(\vartheta_m^{(r)})$,\quad
$\lambda_m^{(r)}\leftarrow-\alpha_m^{(r)}+i\omega_m^{(r)}$
\STATE $\widetilde o_t\leftarrow\max_c o_{t,c}$
\STATE $\chi_t^{(1)}\leftarrow w_t\widetilde o_t$,\quad
$\chi_t^{(2)}\leftarrow w_t$
\STATE $H^{(0)}\leftarrow w\odot
\operatorname{SiLU}(\operatorname{DWConv}(w\odot((o\odot X)W_X^\top)))$
\STATE $U^{(1)}\leftarrow\operatorname{RMSNorm}(H^{(0)})$
\STATE $[E_{\Re}^{(1)}\mid E_{\Im}^{(1)}]\leftarrow U^{(1)}A_1$,\quad
$E^{(1)}\leftarrow E_{\Re}^{(1)}+iE_{\Im}^{(1)}$
\STATE $Z^{(1)}\leftarrow
\operatorname{DirectBankScan}(E^{(1)};\lambda^{(1)},h,\chi^{(1)})$
\STATE $H^{(1)}\leftarrow H^{(0)}+
s_1\odot\left([\Re Z^{(1)}\mid\Im Z^{(1)}]A_1^\top+
\gamma_1\odot U^{(1)}\right)$
\STATE $\phi^{(1)}\leftarrow
\operatorname{MomentReadout}(Z^{(1)};w, \mathcal T)$
\STATE $\overline H^{(1)}\leftarrow w\odot H^{(1)}$
\STATE $U^{(2)}\leftarrow w\odot\operatorname{RMSNorm}\!\left(
\operatorname{SiLU}(\operatorname{DWConv}(\overline H^{(1)}))\right)$
\STATE $[E_{\Re}^{(2)}\mid E_{\Im}^{(2)}]\leftarrow U^{(2)}A_2$,\quad
$E^{(2)}\leftarrow E_{\Re}^{(2)}+iE_{\Im}^{(2)}$
\STATE $Z^{(2)}\leftarrow
\operatorname{CascadedBankScan}(E^{(2)};\lambda^{(2)},h,\chi^{(2)})$
\STATE $\phi^{(2)}\leftarrow
\operatorname{MomentReadout}(Z^{(2)};w, \mathcal T)$
\RETURN $W_{\rm head}[\phi^{(1)}\mid\phi^{(2)}]+b_{\rm head}$
\end{algorithmic}
\end{algorithm}

\section{Theoretical Properties of Modal Compression}
\label{sec:theory}

\subsection{Pole moments as spectral measurements}
On an equal-step, fully observed grid, let
\(V_t^{(y)}\in\R^D\) be the second-order stationary feature process
entering a pole bank for class \(y\), with spectral measure
\(\mathbf F_y\) defined by
\begin{equation}
\Gamma_y(k)=\E[V_{t+k}^{(y)}V_t^{(y)\top}]
=\int_{-\pi}^{\pi}e^{ik\theta}\mathbf F_y(d\theta).
\label{eq:matrix-spectrum}
\end{equation}
We use uncentered moments, so a nonzero feature mean appears as an atom
at frequency zero. For orthonormal analysis directions
\(a,b\in\R^D\), set \(c=a+ib\) and
\(e_t^{(y)}=V_t^{(y)\top}c\), whose scalar spectral measure is
\(\mu_{y,c}(d\theta)=c^\top\mathbf F_y(d\theta)\bar c\).

Write \(p=e^{\lambda h}=\rho e^{i\varphi}\) and
\(\gamma_h(\lambda)=(e^{\lambda h}-1)/\lambda\). Here \(\varphi\)
selects the preferred frequency, while \(\rho<1\) controls the
bandwidth. The frequency response and stationary state are
\begin{equation}
H_{\lambda,h}(e^{i\theta})=
\frac{\gamma_h(\lambda)}{1-pe^{-i\theta}},
\qquad
z_t=\gamma_h(\lambda)\sum_{j=0}^{\infty}p^je_{t-j}.
\label{eq:stationary-pole-state}
\end{equation}
Stable linear filtering gives
\cite{byrnes2000three,georgiou2002statecovariance}
\begin{align}
Q_y(c,\lambda;h)&:=\E|z_t|^2
=|\gamma_h(\lambda)|^2\int_{-\pi}^{\pi}
\frac{\mu_{y,c}(d\theta)}{|e^{i\theta}-p|^2},
\label{eq:pole-energy-spectrum}\\
R_{y,\tau}(c,\lambda;h)&:=\E[z_t\bar z_{t-\tau}]
=|\gamma_h(\lambda)|^2\int_{-\pi}^{\pi}
\frac{e^{i\tau\theta}\mu_{y,c}(d\theta)}
{|e^{i\theta}-p|^2}.
\label{eq:pole-lag-spectrum}
\end{align}
Thus \(Q_y\) is localized spectral mass rather than energy at a single
Fourier frequency. The kernel is centered near \(\varphi\) and sharpens
as \(\rho\) approaches one. Defining
\(P_p(\theta)=(1-|p|^2)/|e^{i\theta}-p|^2\) gives
\begin{equation}
\frac{1-|p|^2}{|\gamma_h(\lambda)|^2}Q_y(c,\lambda;h)
=\int P_p(\theta)\,\mu_{y,c}(d\theta).
\label{eq:poisson-mass}
\end{equation}
Hence modal energy is a positive rescaling of a Poisson-localized
spectral mass. The complex moments \(R_{y,\tau}\) retain
phase-sensitive lag information and provide additional coordinates
for the implemented finite bank.
\subsection{Characteristicness and affine separation}

Define the directional Poisson transform
\begin{equation}
\mathcal P_{\mathbf F}(c,p)
:=
\int_{-\pi}^{\pi}
P_p(\theta)\,c^\top\mathbf F(d\theta)\bar c.
\label{eq:directional-poisson-transform}
\end{equation}
For fixed \(c\), this transform determines the directional spectral
measure; varying the real and imaginary directions recovers the full
matrix-valued measure. Its real-analytic parameter dependence yields
the following result.

\begin{theorem}[Characteristicness and generic separation]
\label{thm:spectral-characteristic}
Let \(D\geq3\), let \(\mathcal V_2(\R^D)\) be the Stiefel manifold of
orthonormal direction pairs, and let \(\Lambda_p\) be a connected open
set of poles satisfying \(0<|p|<1\). Set
\(\Xi=\mathcal V_2(\R^D)\times\Lambda_p\) and, for
\(\xi=((a,b),p)\in\Xi\), define
\(\mathcal P_{\mathbf F}(\xi)=\mathcal P_{\mathbf F}(a+ib,p)\).

If finite matrix-valued spectral measures \(\mathbf F\) and \(\mathbf G\)
satisfy \(\mathcal P_{\mathbf F}(\xi)=\mathcal P_{\mathbf G}(\xi)\) on
a measurable set \(E\subseteq\Xi\) of positive smooth volume, then
\(\mathbf F=\mathbf G\). Consequently, for any pairwise distinct
\(\mathbf F_1,\ldots,\mathbf F_K\), their transform values are pairwise
distinct for almost every \(\xi\in\Xi\). Injectivity also holds at any
fixed damping radius \(\rho\in(0,1)\) when frequency and direction pairs
vary.
\end{theorem}

For \(p=e^{\lambda h}\),
\[
Q_{\mathbf F}(c,\lambda;h)
=
\frac{|\gamma_h(\lambda)|^2}{1-|p|^2}
\mathcal P_{\mathbf F}(c,p).
\]
The positive class-independent factor transfers generic separation to
the modal energies \(Q_y\). Thus a generic mode can separate classes
that have the same feature mean and \(\Gamma(0)\) but differ in some
\(\Gamma(k)\), \(k\ne0\).

\begin{corollary}[Affine realizability and sequence-length consistency]
\label{cor:affine-consistency}
Fix a mode whose class energies \(Q_1,\ldots,Q_K\) are pairwise
distinct, and set \(q_y=\log(1+Q_y)\). The affine logits
\(\ell_y(q)=2q_yq-q_y^2\) classify every population prototype exactly,
since
\[
\ell_i(q_i)-\ell_j(q_i)=(q_i-q_j)^2>0
\qquad (i\ne j).
\]

Suppose each class feature process is stationary ergodic with finite
second moments. For the zero-initialized state trajectory, define
\[
\widehat Q_T=\frac1T\sum_{t=1}^T|z_t|^2,\qquad
\widehat y_T=\operatorname*{arg\,max}_{j}
\ell_j\!\left(\log(1+\widehat Q_T)\right).
\]
Then
\begin{equation}
\widehat Q_T\xrightarrow[T\to\infty]{\mathrm{a.s.}}Q_y,
\qquad
\Pr(\widehat y_T\ne y)\longrightarrow0.
\label{eq:ergodic-consistency}
\end{equation}
Under bounded excitation and summable energy autocovariance, the
supplementary proof gives an \(O(T^{-1})\) error-probability bound.
\end{corollary}

When \(2M\leq D\), the separating direction pair can be completed to
the model's semi-orthogonal \(M\)-mode frame; finite-bank margins are
examined below.

\subsection{Information-preserving moment conditioning}

We show that the logarithmic conditioning used in the modal descriptor
preserves moment information while controlling perturbations.

\begin{lemma}[Moment conditioning]
\label{lem:radial-log-nonexpansive}
For
\[
u=\bigl(u_0,(u_\tau)_{\tau\in\mathcal T}\bigr)
\in\R_{\geq0}\times\C^{|\mathcal T|},
\]
define
\[
\Psi(u)=
\Bigl[
\log(1+u_0)
\ \Big|\
\bigl(\operatorname{rlog}(u_\tau)\bigr)_{\tau\in\mathcal T}
\Bigr],
\]
identifying complex coordinates with pairs of real coordinates. Then
\(\Psi\) is injective and, for all \(u,v\),
\begin{equation}
\|\Psi(u)-\Psi(v)\|_2^2
\leq
\langle\Psi(u)-\Psi(v),u-v\rangle
\leq\|u-v\|_2^2.
\label{eq:radial-log-nonexpansive}
\end{equation}
Moreover, if \(\|u\|_2,\|v\|_2\leq U\), then
\begin{equation}
\frac{1}{1+U}\|u-v\|_2
\leq\|\Psi(u)-\Psi(v)\|_2
\leq\|u-v\|_2.
\label{eq:radial-log-bilipschitz}
\end{equation}
These statements extend to direct sums over all modes and both banks.
\end{lemma}
The proof and a trajectory-to-descriptor perturbation bound are in the supplement.

\subsection{Stable dynamics and bounded prediction}

Unlike the spectral results, the following bounds require only stable
poles, bounded normalized excitation, and \(h_t\geq0\), and apply to
both banks.

\begin{proposition}[Modal input-to-state bound]
\label{prop:stability}
If \(\alpha_\star=\min_m\alpha_m>0\), \(\|e_t\|_2\leq V\), and
\(T_n=\sum_{t=1}^n h_t\), then
\begin{equation}
\|z_n\|_2
\leq
e^{-\alpha_\star T_n}\|z_0\|_2
+\frac{V}{\alpha_\star}(1-e^{-\alpha_\star T_n}).
\end{equation}
\end{proposition}
Thus, for \(z_0=0\), \(\|z_n\|_2\leq V/\alpha_\star\) independently
of token count. Modewise and synthesized-stream refinements are given
in the supplement.

\begin{corollary}[Bounded prediction interface]
\label{cor:descriptor-bound}
Let \(\eta_r\in\R^D\) be the RMSNorm gain of bank
\(r\in\{1,2\}\). Semi-orthogonality and RMS normalization imply
\[
\|e_t^{(r)}\|_2\leq
\kappa_r:=\sqrt D\|\eta_r\|_\infty,\qquad
B_{r,m}:=\frac{\kappa_r}{\alpha_m^{(r)}}.
\]
Define
\[
C_g^2
=
(1+|\mathcal T|)
\sum_{r=1}^2\sum_{m=1}^M
\log^2(1+B_{r,m}^2).
\]
With zero-initialized stable states,
\begin{equation}
\begin{aligned}
\|g(X)\|_2 &\leq C_g,\\
\|f_\Theta(X)\|_2
&\leq \|W_{\rm head}\|_2C_g+\|b_{\rm head}\|_2,
\end{aligned}
\label{eq:prediction-interface-bound}
\end{equation}
uniformly over finite raw-input amplitudes, sequence lengths, validity
patterns, and nonnegative interval grids in exact arithmetic.
\end{corollary}

This independence from raw-input amplitude and upstream operator norms
follows because RMS normalization precedes each scan and the head uses
only modal moments.

\section{Experiments}
\label{sec:evaluation}

\subsection{Evaluation protocol}

We evaluate a fixed registry of \BenchmarkTaskCount{} public tasks:
\BenchmarkUCRTaskCount{} UCR datasets \cite{dau2019ucr},
\BenchmarkUEAFaultTaskCount{} multivariate sequence and vibration/fault tasks
including UEA \cite{bagnall2018uea}, and \BenchmarkGeneralTaskCount{}
additional sequence, ECG, clinical/activity, and forecasting tasks. The
registry spans univariate and multivariate signals, regular and irregular
observations, and classification, multilabel, and forecasting objectives; it
is coverage-oriented rather than archive-exhaustive
\cite{bagnall2017bakeoff}. Each task is trained independently on its predefined
split, with no pooled multi-dataset training. All ten families complete every
task without missing or imputed cells. For each task--family pair, one of 18
family-native candidates is selected using TRAIN-derived validation, frozen
before TEST, and evaluated over five final-training seeds. Full registry,
split, preprocessing, and selection appear in the
supplement.
\subsection{Predictive rank, capacity, and systems efficiency}
\label{sec:final-runtime}

The left panel of Figure~\ref{fig:efficiency-summary} summarizes predictive
rank and model capacity on the \BenchmarkCompleteTaskCount{} tasks with
complete results for all ten families. \model{} has mean rank $3.97$, 14
Top-1 placements, and 43 Top-3 placements.

\begin{figure}[t]
\centering
\includegraphics[width=\columnwidth]{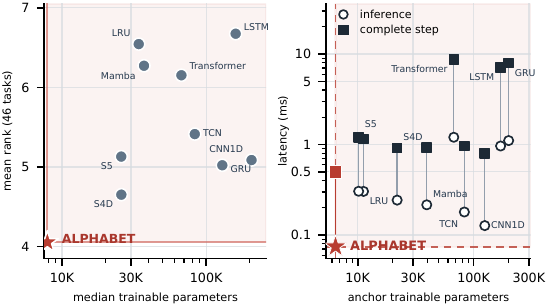}
\caption{Predictive and systems efficiency.
Left: Mean TEST rank versus median trainable parameters over the complete
ten-family subset.
Right: FP32 inference and complete-step latency at the same-width
\(D=64\) anchor (\(B=32,T=512\)).}
\label{fig:efficiency-summary}
\end{figure}

We optimize \model{} with parity-gated Triton kernels and shape-static CUDA
Graph capture, then measure steady-state FP32 inference and complete training
steps against nine baselines on one RTX~4090. We use each family's fastest
public implementation under an identical measurement harness, enabling
compiled kernels when provided. At the evaluation-width anchor ($D{=}64$,
$B{=}32$, $T{=}512$), \model{} is the smallest and fastest model tested, with
6,437 trainable parameters, $5.02\times$ faster inference and $3.93\times$
faster complete steps on average. Across the broader audited grid of batch
sizes and sequence lengths, this efficiency advantage generally persists,
although its magnitude varies with execution phase, shape, and baseline. Full
shape-wise results, timing methodology, and numerical-parity checks are
reported in the supplementary systems evaluation.

\subsection{Controlled spectral recovery under structural compression}
Theorem~\ref{thm:spectral-characteristic} establishes continuum identification
and generic population separation for fixed, spectrally distinct feature
processes. The remaining finite-bank question is whether end-to-end learning
realizes a margin large enough to estimate and exploit in practice. Moreover,
because almost every admissible mode separates any fixed finite collection of
distinct spectra, a covering fixed random bank can plausibly remain competitive
with learned poles. The following diagnostics test these predictions using
TRAIN-derived validation only.

We use a stronger moment-matched control. Class 0 is unit-variance white
Gaussian noise; class 1 is a sparse stationary MA(5) process with nonzero
coefficients only at lags 0 and 5, $X_t=aZ_t+bZ_{t-5}$, where
$a^2+b^2=1$ and $2ab=\epsilon$. Their population autocovariances agree exactly
for lags 0 through 4 and first differ at lag 5. Moment matching removes
population separation from the explicitly supplied lag-$0{:}4$ coordinates,
turning Figure~\ref{fig:finite-bank} into a controlled test of higher-lag
spectral access under compression. Consistent with this population match, the
raw $\widehat\Gamma(0{:}4)$ control remains at chance empirically, while the
oracle-informed $\widehat\Gamma(0{:}8)$ control nearly reaches the Bayes
ceiling. The full \model{} descriptor rises from $.567$ to $.991$ as the spectral
separation $\epsilon$ increases, without being given the distinguishing lag.

At $\epsilon=.4$, the 5.7k-parameter descriptor (the diagnostic-task configuration) reaches $.856$ balanced
accuracy, compared with $.787$ for LRU and $.750$ for S4D under the evaluated
configurations. Increasing $T$ closes the remaining gap: Bayes-normalized
excess-over-chance accuracy rises from $91.3\%$ at $T=128$ to $99.7\%$ at
$T=1024$. This supports a finite-sample explanation for the short-sequence
deficit; it does not establish the $O(T^{-1})$ rate or exact asymptotic
equivalence to Bayes. 

\begin{figure}[t]
\centering
\includegraphics[width=\columnwidth]{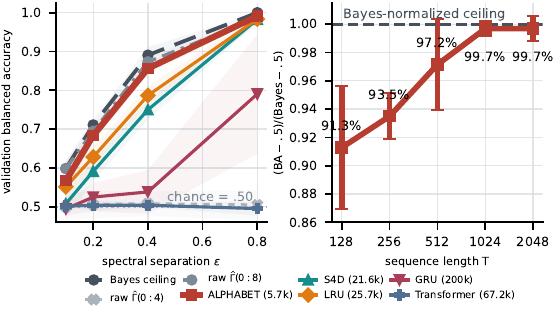}
\caption{Controlled spectral recovery under structural compression.
Left: Validation balanced accuracy versus spectral separation over
five seeds. Right: Bayes-normalized excess-over-chance accuracy
versus sequence length with paired 95\% Student-\(t\) intervals. All results
use TRAIN-derived validation only.}
\label{fig:finite-bank}
\end{figure}

\subsection{Focused structural ablation}
We isolate the two-scan design under a doubly matched control: $M'=32$
gives the one-scan baseline the same 224 real coordinates as the two $7M$
summaries, and $D'=72$ equalizes trainable parameters on every task. It is
not a smaller model, but the same budget spent on a single bank. The
audit contains 17 UCR tasks, each run over five TRAIN-derived validation
seeds with no TEST access.
  
\begin{table}[!b]
\centering
\small
\begin{tabular}{@{}lrrrr@{}}
\toprule
Variant & BA & $\Delta$ (pp) & Seed SD & Med. params \\
\midrule
Full two-scan    & \textbf{.863} & ---     & \textbf{.061} & 5,698 \\
Matched one-scan & .830          & $-3.33$ & .074          & 5,698 \\
No synthesis     & .847          & $-1.67$ & .071          & 5,570 \\
Fixed poles      & .847          & $-1.67$ & .074          & 5,634 \\
Energy-only MLP  & .808          & $-5.48$ & .085          & 5,705 \\
\bottomrule
\end{tabular}
\caption{Matched structural ablation on 17 UCR tasks and five
TRAIN-derived validation seeds. BA and seed SD are averaged across tasks,
parameter counts are task-wise medians, and $\Delta$ is relative to the full
model.}
\label{tab:structural-ablation}
\end{table}

The full model exceeds the capacity-matched one-scan control by $3.33$ points
(task-paired 95\% CI
$[0.05,6.62]$; W/T/L $10/3/4$). It exceeds the no-synthesis and energy-only
controls by $1.67$ ($[-0.53,3.88]$) and $5.48$ points
($[0.74,10.21]$), respectively. The energy-only MLP remains matched within
$0.2\%$ of the full model's trainable count, so extra head capacity does not
generally recover the lag-coordinate information. Fixed poles trail by
$1.67$ points but remain statistically competitive ($[-1.01,4.36]$), as the
genericity result predicts. None of the four reported contrasts has
Holm-adjusted $p<.05$; the results therefore indicate a task-dependent
two-scan advantage without treating any component as universally necessary.

\subsection{Pole-level prediction audit}
\begin{figure}[t]
\centering
\includegraphics[width=\columnwidth]{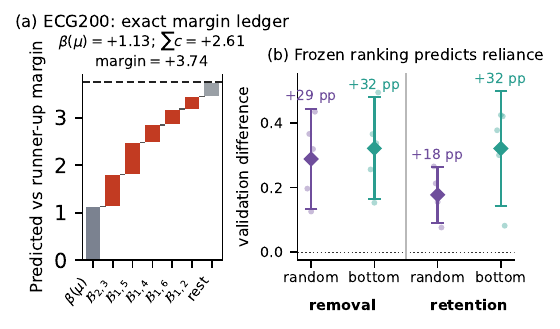}
\caption{Pole attribution audit.
Left: Exact ECG200 margin contributions from the top five poles and remainder.
Right: Validation contrasts for frozen TRAIN-ranked top-eight
removal/retention against random and bottom controls over 17 datasets. Points
denote datasets, diamonds means, and bars paired 95\% $t$ intervals.}
\label{fig:pole-attribution}
\end{figure}
For predicted class $y$ and runner-up $j$, the affine head gives an exact
per-pole margin decomposition. Let $I_{b,m}$ index the seven coordinates of
bank $b$ and pole $m$, and let $\mu$ be the optimization-fold mean descriptor:
\[
\begin{aligned}
c_{b,m}
&=(W_y-W_j)_{I_{b,m}}^\top(g-\mu)_{I_{b,m}},\\
\beta_{yj}(\mu)
&=b_y-b_j+(W_y-W_j)^\top\mu,\\
\ell_y-\ell_j
&=\beta_{yj}(\mu)+\sum_{b,m}c_{b,m}.
\end{aligned}
\]
This algebraic identity reconstructs FP32 margins within $8.6\times10^{-6}$.
Unlike an unstructured coordinate decomposition, \model{} groups its compact
descriptor explicitly by bank and pole.

Removal replaces a pole's seven descriptor coordinates with their
optimization-fold means, leaving all other coordinates intact. Across 17 datasets and 85 checkpoints, optimization-fold pole rankings
transfer to validation with mean rank correlation $.990$. Removing the top
eight poles reduces balanced accuracy by $26.9$ points more than random
removal (paired 95\% CI $[19.0,34.7]$), while retaining only the top eight
improves full-model agreement (prediction agreement with the unmasked model) by $23.3$ points over random retention
($[16.6,30.0]$). The top-versus-bottom contrasts are $29.4$
($[21.2,37.7]$) and $43.7$ points ($[33.8,53.6]$), respectively
(Figure~\ref{fig:pole-attribution}).

The cascaded bank accounts for $63.5\%$ of validation attribution and $64.4\%$
of top-eight positions, with at least one cascaded pole in 84 of 85
checkpoints. Removing all 16 cascaded poles costs $13.6$ points more than a
size-matched random mask ($[6.8,20.4]$), whereas cascaded-only retention
remains borderline and inconclusive ($4.3$ points, $[-0.04,8.65]$;
$p=.052$). Thus the audit shows systematic use of the cascaded bank
without treating mean masking as proof of architectural necessity.

\subsection{Corruption audit}
\label{sec:corruption-audit}

\begin{figure}[t]
\centering
\includegraphics[width=.8\columnwidth]{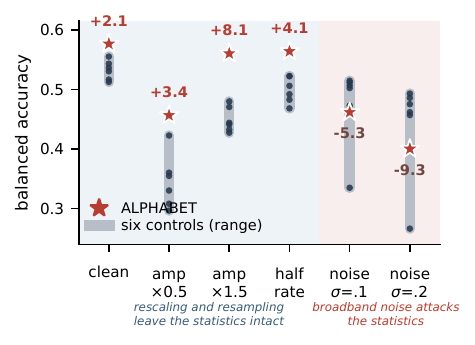}
\caption{Corruption audit on 17 UCR datasets and five seeds. Gray bars span
the nine controls' across-task means, stars show
\model{}'s across-task mean, and annotations give its gap to the best control
mean. These synthetic corruptions are not domain-level OOD.}
\label{fig:corruption}
\end{figure}

\paragraph{Mechanistic prediction.}
The modal equations predict an asymmetric response to these corruptions. If a
residual global scale $s>0$ reaches a linear modal path, then
\[
Q_m^{(s)}=s^2Q_m,
\qquad
R_{m,\tau}^{(s)}=s^2R_{m,\tau}.
\]
Scaling therefore preserves modal phase, while the radial-log map compresses
the magnitude change to $\log(1+s^2|R_{m,\tau}|)$. RMS normalization before
each scan further attenuates global scale changes, although the nonlinear lift
prevents exact invariance.

Writing the exact interval update as
\[
\Phi_h(z,e)=e^{\lambda h}z+\gamma_h(\lambda)e,
\]
constant excitation gives
\[
\Phi_{h_2}\!\left(\Phi_{h_1}(z,e),e\right)=\Phi_{h_1+h_2}(z,e).
\]
Elapsed-time transport is therefore resolution-consistent under a
zero-order-hold signal, although discarded samples, aliasing, and token-lag
readout prevent exact invariance after downsampling.

Additive white noise acts differently. Under the feature-space idealization
$V_t=S_t+N_t$, with independent isotropic white noise
$\operatorname{Cov}(N_t)=\sigma^2I$, the spectral identities give
\[
\begin{aligned}
Q_{S+N} &= Q_S+\nu, &
R_{S+N,\tau} &= R_{S,\tau}+\nu p^\tau,\\
\nu &=
\sigma^2\lVert c\rVert_2^2
\frac{|\gamma_h(\lambda)|^2}{1-|p|^2}>0,
&
p&=e^{\lambda h}.
\end{aligned}
\]
Thus white noise creates a pole-dependent energy floor and contaminates the
lag moments; radial-log conditioning compresses but does not remove these
terms. The model should therefore be comparatively more robust to amplitude
and sampling-resolution shifts than to broadband additive noise.

We test this prediction against nine validation-frozen controls under five
deterministic TEST corruptions across 17 UCR tasks
(Figure~\ref{fig:corruption}). In across-task means, \model{} leads the best
control by $10.2$ and $9.4$ points at amplitude $\times0.5$ and $\times1.5$,
and by $1.3$ points after half-rate resampling. Dataset-paired intervals
against each task's best control span zero for both amplitude shifts, so these
are aggregate advantages rather than universal task-wise wins. Under additive
noise the aggregate ordering reverses: \model{} trails by $1.6$ and $4.3$
points at $\sigma=.1$ and $.2$. The stricter task-paired gaps are $-7.2$
($[-13.6,-0.8]$) and $-10.4$ points ($[-17.8,-3.0]$), respectively.

This reversal matches the predicted asymmetry, so we report additive noise as
a distinct failure mode rather than averaging all corruptions into a generic
robustness score.
\section{Discussion and Limitations}
\label{sec:discussion}
The spectral bridge has a precise domain: identification and affine
consistency hold with the encoder fixed and the feature process stationary,
fully observed, and equal-step. Two boundaries matter. First, the object
identified is the spectrum of the learned feature process, not of the raw
input; the theory characterizes what the modal readout preserves about the
encoder's output, not what the encoder discards. Second, the guarantee is
generic rather than constructive: it does not show that optimization finds a
separating mode, and the learned-bank diagnostics supply only empirical
evidence that it can. Outside this domain, the interval-aware recurrence and
the descriptor and logit bounds remain valid on irregular grids without
stationarity; higher-order non-Gaussian class differences fall outside the
second-order theory entirely.

The empirical boundaries are equally concrete. The clearest is additive
noise, where broadband energy contaminates the modal statistics at both
evaluated noise levels; we would not recommend this representation in
low-SNR settings. Across 17 tasks, the two-scan advantage stays positive but
no ablation contrast survives Holm correction, fixed poles remain
competitive, and attribution in the fixed-pole variant is similarly
concentrated---so the evidence supports an auditable stable-bank interface
more strongly than uniquely meaningful learned pole placement.

A post-hoc shrinkage rule that suppresses weak lag moments failed to
transfer to dataset-disjoint validation (supplement), consistent with a
descriptor-level ambiguity: weak lag coherence may be nuisance noise or weak
class-bearing signal, and an unconditional rule cannot separate them. The
natural response---input-conditioned, pole-response-aware debiasing of $Q_m$
and $R_{m,\tau}$---must be frozen and tested prospectively.

\section{Conclusion}
\label{sec:conclusion}

\model{} provides evidence that competitive temporal prediction can be
achieved without a large or opaque backbone. Two banks of stable Laplace poles
compress temporal history into a fixed set of modal statistics whose spectral
meaning is characterized by the theory, while predictions decompose exactly
into per-pole contributions. Under the matched evaluation settings considered
in this work, \model{} achieves substantially lower latency than all evaluated
baselines. The audits show that this transparency is functional rather than
nominal, and demonstrate that compactness, speed, and an auditable
representation can be achieved simultaneously.
\clearpage
\appendix
\input{sections/appendix/technical_supplement_exact}

\bibliography{references}

\end{document}

%% file: sections/appendix/technical_supplement_exact.tex
\section{Proofs and Mathematical Boundaries}
\label{app:proofs}

\subsection{Radial-log geometry}

\begin{lemma}[Injective non-expansive moment conditioning]
\label{lem:app-radial-log-conditioning}
For a finite lag set \(\mathcal T\), let \(\Psi_{\mathcal T}\) apply
\(\log(1+\cdot)\) to the energy and
\(\operatorname{rlog}(u)=\log(1+|u|)u/|u|\) to each complex lag, with
\(\operatorname{rlog}(0)=0\). Identifying \(\C\) with \(\R^2\),
\(\Psi_{\mathcal T}\) is injective and maps zero to zero. For moment
descriptors \(r,s\in\R_{\geq0}\times\C^{|\mathcal T|}\), writing
\(\Delta_\Psi=\Psi_{\mathcal T}(r)-\Psi_{\mathcal T}(s)\), it satisfies
\begin{align}
\|\Delta_\Psi\|_2^2
&\leq\langle\Delta_\Psi,r-s\rangle
\leq\|r-s\|_2^2,
\label{eq:app-radial-log-nonexpansive}
\end{align}
Moreover, for any \(U\geq 0\) such that \(\|r\|_2,\|s\|_2\leq U\),
\begin{align}
\frac{1}{1+U}\|r-s\|_2
&\leq\|\Delta_\Psi\|_2
\leq\|r-s\|_2.
\label{eq:app-radial-log-bilipschitz}
\end{align}
Here \(U\) is the radius of the bounded moment-descriptor region,
not a bound on the raw input. Both statements extend by direct sums.
\end{lemma}

\paragraph{Proof.}
Identify \(\C\) with \(\R^2\) and write
\(\psi(u)=\log(1+|u|)u/|u|\). Its radial and tangential Jacobian
eigenvalues are
\[
\frac{1}{1+|u|},\qquad \frac{\log(1+|u|)}{|u|}.
\]
They extend to one at zero and lie in \([(1+U)^{-1},1]\) on
\(|u|\leq U\); the scalar energy coordinate has the same bounds. Since the
Euclidean ball is convex, \(s+td\) remains in this ball for \(t\in[0,1]\).
For \(d=r-s\), \(\Delta=\Psi_{\mathcal T}(r)-\Psi_{\mathcal T}(s)\), and the
block Jacobian \(J_t\) at \(s+td\),
\[
\|\Delta\|_2^2
\leq\int_0^1\|J_td\|_2^2dt
\leq\int_0^1\langle J_td,d\rangle dt
=\langle\Delta,d\rangle
\leq\|d\|_2^2.
\]
The lower Jacobian bound and Cauchy--Schwarz give
Equation~\eqref{eq:app-radial-log-bilipschitz}; strict radial monotonicity
(inverse radius \(e^r-1\)) gives injectivity. Direct sums complete the proof.

\subsection{Spectral embedding, generic separation, and consistency}
\label{app:spectral-proofs}

Under the main paper's equal-step, fully observed population model, its
energy and lag identities reduce the pole statistics to Poisson-kernel
integrals. Positive rescaling and Lemma~\ref{lem:app-radial-log-conditioning}
therefore preserve identification. Let
$d(\xi)=\mathcal P_{\mathbf F}(\xi)-\mathcal P_{\mathbf G}(\xi)$ on the connected
real-analytic parameter manifold $\Xi$, equipped with its smooth volume
measure. If $d$ vanishes on a positive-volume set, real analyticity implies
$d\equiv0$. For $c=a+ib$, let
$\mu_{\Delta,c}(A)=c^\top(\mathbf F-\mathbf G)(A)\bar c$ be the corresponding
finite complex Borel measure. Its Poisson integral is then zero. At any fixed radius $r\in(0,1)$,
\begin{equation}
P_{re^{i\varphi}}(\theta)
=\sum_{k\in\mathbb Z}r^{|k|}e^{ik(\theta-\varphi)}.
\label{eq:poisson-fourier}
\end{equation}
All multipliers are nonzero, so every Fourier--Stieltjes coefficient of
$\mu_{\Delta,c}$ vanishes and hence $\mu_{\Delta,c}=0$.

To recover the matrix measure, write
$\mathbf\Delta(A)=S+iA_0$ with $S$ real symmetric and $A_0$ real
skew-symmetric. For every orthonormal $a,b$ and $c=a+ib$,
\begin{equation}
c^\top\mathbf\Delta(A)\bar c
=a^\top Sa+b^\top Sb+2a^\top A_0b=0.
\label{eq:paired-matrix-projection}
\end{equation}
Subtracting the equation with $b$ replaced by $-b$ gives
$a^\top A_0b=0$ for every orthonormal pair $a,b$, and hence $A_0=0$.
The remaining identity is $a^\top Sa+b^\top Sb=0$. In an eigenbasis of $S$, this is
$\sigma_i+\sigma_j=0$ for $i\ne j$; three indices (and hence $D\ge3$)
force $S=0$. Thus $\mathbf F=\mathbf G$. The same argument applies at a
fixed damping radius when direction and frequency vary.

For pairwise distinct class spectra, each difference
$d_{ij}(a,b,p)=\mathcal P_{\mathbf F_i}(a+ib,p)-
\mathcal P_{\mathbf F_j}(a+ib,p)$ is analytic on the connected parameter
manifold and not identically zero. Its zero set therefore has measure zero;
a finite union over class pairs proves simultaneous single-mode separation
for almost every direction and pole.

When \(2M\leq D\), the separating pair can be completed by Gram--Schmidt to
an admissible $M$-mode frame. This is an existence statement, not a claim
that optimization finds that frame.

For a stationary ergodic process satisfying
\(\mathbb E|z_0^\star|^2<\infty\), Birkhoff's theorem gives
\(T^{-1}\sum_{t=1}^T|z_t^\star|^2\to Q_y\) almost surely. The zero-initialization
transient satisfies \(z_t^{(0)}-z_t^\star=-p^tz_0^\star\) and contributes
vanishing average energy. Thus \(\widehat Q_T\to Q_y\); a positive minimum
class margin gives eventual correctness and
\(\Pr(\widehat y_T\ne y)\to0\).

A finite-length bound requires bounded excitation and summable energy
autocovariance. With $|e_t|\leq B_e$ and $r=|p|<1$,
\begin{align}
Z&=\frac{|\gamma_h(\lambda)|B_e}{1-r},\\
\Gamma_y^{(E)}&=\operatorname{Var}(|z_0^\star|^2)
+2\sum_{k\geq1}\left|\operatorname{Cov}(|z_0^\star|^2,|z_k^\star|^2)\right|<\infty.
\end{align}
The transient discrepancy is bounded by
\begin{equation}
b_T=\frac{2Z^2}{T(1-r)}.
\end{equation}
With class-$y$ radius
\begin{equation}
\varrho_y
=
\min_{j\ne y}
\left|
\sqrt{(1+Q_y)(1+Q_j)}-(1+Q_y)
\right|,
\end{equation}
misclassification requires an additional deviation of at least
$\varrho_y-b_T$. Chebyshev's inequality gives, when $b_T<\varrho_y$,
\begin{equation}
\Pr_y(\widehat y_T\ne y)
\leq\frac{\Gamma_y^{(E)}}{T(\varrho_y-b_T)^2}.
\label{eq:finite-length-classification}
\end{equation}

\begin{proposition}[Finite-bank boundary]
\label{prop:finite-bank-boundary}
For any fixed finite pole bank and finite collection of energy and raw complex-autocorrelation moments, there exist
distinct strictly positive, real-analytic, full-support scalar spectral densities that produce
identical modal population descriptors.
\end{proposition}
\begin{proof}
The finite moments define finitely many linear functionals on any
sufficiently large finite-dimensional space $\mathcal H$ of real even
trigonometric polynomials. A nonzero common-kernel element $h$ yields
$f_\pm=c\pm\epsilon h$ with identical measured moments for sufficiently
small $\epsilon>0$. These are distinct, strictly positive analytic spectra,
so the finite learned bank is not globally injective.
\end{proof}
This proposition rules out global injectivity over the full class of analytic
spectra, but does not preclude separation on finite task families or under the
generic conditions stated above.

\subsection{Cumulative-step and prediction-interface bounds}

For the exact-ZOH recurrence in the main paper, write
$T_n=\sum_{t=1}^nh_t$ and
$\gamma_h(\lambda)=(e^{\lambda h}-1)/\lambda$. Variation of constants
\cite{chen1999linear} gives the exact unroll
\begin{equation}
 z_{m,n}=e^{\lambda_mT_n}z_{m,0}
 +\sum_{j=1}^ne^{\lambda_m(T_n-T_j)}
 \gamma_{h_j}(\lambda_m)\chi_j^{(r)}e_{m,j}^{(r)}.
 \label{eq:app-exact-profile-unroll}
\end{equation}
Semi-orthogonality gives \(\|vA_r\|_2\leq\|v\|_2\). Applying this contraction to
Equation~\eqref{eq:app-exact-profile-unroll}, using
$|\chi_j^{(r)}|\leq1$, $\|e_j^{(r)}\|_2\leq V$, and
$|\gamma_h(\lambda_m)|\leq\int_0^h e^{-\alpha_m u}\,du$, and summing adjacent
physical-time intervals gives the main-paper bound
\[
\|z_n\|_2\leq e^{-\alpha_\star T_n}\|z_0\|_2+
\frac{V}{\alpha_\star}(1-e^{-\alpha_\star T_n}).
\]
This argument requires only $h_t\geq0$ and stable poles, not stationarity or equal spacing.

For FP32, the implementation substitutes $h$ for $\gamma_h(\lambda)$ when
$|\lambda h|<\delta_0=10^{-6}$. Its coefficient error is at most
$\tfrac12\delta_0e^{\delta_0}h$, and the same state bound holds after multiplying the input
term by $\chi_0=\delta_0/(1-e^{-\delta_0})$. This finite-precision qualification does not
strengthen the exact-arithmetic proposition.

\begin{proposition}[Synthesized-stream bound]
\label{prop:app-synthesized-stream-bound}
Assume $\|H^{(0)}_n\|_2\leq H$, $\|U^{(1)}_n\|_2\leq V$, and input gates have magnitude at most
one. In this proposition, \(\alpha_\star\) and \(z_0\) refer to the direct
bank. Define
\begin{equation}
B_n=e^{-\alpha_\star T_n}\|z_0\|_2+
\frac{V}{\alpha_\star}(1-e^{-\alpha_\star T_n}).
\end{equation}
Then the synthesized direct-bank trajectory satisfies
\begin{equation}
\|H^{(1)}_n\|_2\leq H+\|s_1\|_\infty
\bigl(B_n+\|\gamma_1\|_\infty V\bigr).
\label{eq:app-synthesized-stream-bound}
\end{equation}
\end{proposition}

\begin{proof}
Semi-orthogonality bounds the excitation by \(V\) and the modal state by
\(B_n\); the synthesis triangle inequality proves
Equation~\eqref{eq:app-synthesized-stream-bound}. Because of the identity
residual, this internal-stream bound remains input-relative. The FP32 branch
replaces \(B_n\) by
$B_n^{\mathrm{FP32}}=e^{-\alpha_\star T_n}\|z_0\|_2+
\chi_0V(1-e^{-\alpha_\star T_n})/\alpha_\star$.
\end{proof}

For bank \(r\), RMS normalization and semi-orthogonal analysis give
\[
\|e_t^{(r)}\|_2\leq\kappa_r:=\sqrt D\|\eta_r\|_\infty.
\]
With zero initialization, the modewise form of the state bound gives
$|z_{m,t}^{(r)}|\leq B_{r,m}:=\kappa_r/\alpha_m^{(r)}$. Consequently,
$0\leq R_{m,0}^{(r)}\leq B_{r,m}^2$, and Cauchy--Schwarz gives
$|R_{m,\tau}^{(r)}|\leq B_{r,m}^2$ for every selected token lag. Each of the
$1+|\mathcal T|$ radial-log blocks therefore contributes at most
$\log^2(1+B_{r,m}^2)$ per mode. Summing over both banks proves
\[
\|g(X)\|_2^2\leq
(1+|\mathcal T|)\sum_{r=1}^2\sum_{m=1}^M
\log^2(1+B_{r,m}^2)=C_g^2.
\]
The affine head therefore obeys
$\|f_\Theta(X)\|_2\leq\|W_{\rm head}\|_2C_g+\|b_{\rm head}\|_2$.
Under the stated normalization, stability, and parameter bounds, the descriptor norm is
independent of raw-input amplitude and token count; validity masks with no
eligible pairs contribute zero by the denominator convention in the main-paper moment definition.

\begin{proposition}[Trajectory-to-descriptor perturbation]
\label{prop:app-trajectory-perturbation}
For the implemented lag set $\mathcal T=\{1,2,4\}$ and bank $r\in\{1,2\}$,
suppose two modal trajectories use the same validity weights and satisfy
$\|z^{(r)}_t\|_2,\|z^{\prime(r)}_t\|_2\leq B_r$ and
$\|z^{(r)}_t-z^{\prime(r)}_t\|_2\leq\varepsilon_r$ for every token. Then their joint radial-log
descriptors obey
\begin{equation}
\|g-g'\|_2
\leq4\sqrt{B_1^2\varepsilon_1^2+B_2^2\varepsilon_2^2}.
\label{eq:app-trajectory-descriptor-perturbation}
\end{equation}
Consequently the affine scores differ by at most
\[
\|f_\Theta-f'_\Theta\|_2
\leq\|W_{\rm head}\|_2\|g-g'\|_2.
\]
\end{proposition}
\begin{proof}
Componentwise squaring and the Hadamard-product inequality give
\[
\begin{aligned}
\bigl\||z_t|^2-|z'_t|^2\bigr\|_2
&\leq2B_r\varepsilon_r,\\
\|z_t\odot\bar z_{t-\tau}-z'_t\odot\bar z'_{t-\tau}\|_2
&\leq2B_r\varepsilon_r.
\end{aligned}
\]
Averaging and Lemma~\ref{lem:app-radial-log-conditioning} preserve these bounds,
so each block differs by at most \(2B_r\varepsilon_r\). Concatenating one energy
block and three complex-lag blocks gives
\[
\|g^{(r)}-g^{\prime(r)}\|_2\leq4B_r\varepsilon_r.
\]
Concatenating the two banks proves
Equation~\eqref{eq:app-trajectory-descriptor-perturbation}; a coordinatewise
bound \(\varepsilon_r^{\rm coord}\) implies
\(\varepsilon_r=\sqrt M\,\varepsilon_r^{\rm coord}\).
\end{proof}

\section{Model Implementation}
\label{app:implementation-details}

The two banks use
\(A_r=[A_{r,\Re}\mid A_{r,\Im}]\), \(A_r^\top A_r=I_{2M}\), maintained by a
matrix-exponential Stiefel parameterization \cite{lezcano2019orthogonal}.
Independently for each bank,
\begin{align*}
\alpha_m&=10^{-3}+(2-10^{-3})\operatorname{sigmoid}(\delta_m),\\
\omega_m&=\pi\tanh(\vartheta_m),
\end{align*}
initialized with \(\delta=\operatorname{linspace}_M(-3,1)\) and
\(\vartheta=\operatorname{atanh}(\operatorname{linspace}_M(0,.75))\); modal
states start at zero.

Each local lift is a centered depthwise kernel of size 5 and dilation 4,
followed by learned bias and SiLU \cite{elfwing2018silu}. Only the direct bank
synthesizes through \(A_1^\top\). With \(\mathcal T=\{1,2,4\}\), the affine
head has \(K(14M+1)\) parameters and no pooled \(D\)-dimensional path.

On equal-step inputs, \(\mathcal T=\{1,2,4\}\) denotes token offsets shared by
both banks. When \texttt{time\_delta} is supplied, the implementation switches
both local lifts and modal readout to fixed physical-time offsets in the
normalized units of \texttt{time\_delta}: local kernels sample the
piecewise-linear signal interpolant, and modal moments query/interpolate states
at \(t-\tau\). The exact-ZOH transition uses the same elapsed-time metadata.
Validity masks gate the interpolation support, and an empty support set returns
zero. Thus the equal-step equations are the unit-grid special case, whereas
irregular runs use physical-time interpolation/quadrature.

For fixed \(D,M,K\), the two width-5 lifts cost \(O(TD)\), the two diagonal pole
recurrences and direct-bank synthesis cost \(O(TMD)\), and the fixed lag set costs
\(O(TM)\). The terminal affine head costs \(O(MK)\). Memory can be streamed with
the current modal states and fixed-lag buffers, so the sequence-dependent
compute is \(O(T)\); the implemented training path may retain trajectories for
automatic differentiation without changing that arithmetic complexity.

\subsection{Systems efficiency audit}
\label{app:efficiency}
\label{app:optimization-details}
\subsubsection{Final radial-log kernel specialization}
\label{app:final-kernel-specialization}

\paragraph{Scope and dispatch.}
The optimized target is the final radial-log affine \model{} with
\((C,D,M,K)=(2,64,16,5)\). The specialized path is used only for
no-gradient CUDA inference with static poles, FP32 tensors,
\(B\in\{32,64\}\), and \(1\leq T\leq2048\). Unsupported shapes, masks,
variable-step metadata, and non-FP32 inputs use the reference graph.
Framework-level TF32 is disabled; the packed projection kernel explicitly
uses Triton's \texttt{tf32x3} decomposition and is validated numerically
against the FP32 reference.

\paragraph{Associative time-parallel scan.}
For one mode, write a recurrence step as
\[
\begin{aligned}
g_{m,t}(z)&=a_{m,t}z+b_{m,t},\\
(a_2,b_2)\star(a_1,b_1)
&=(a_2a_1,\;a_2b_1+b_2).
\end{aligned}
\]
Because \(\star\) is associative, a parallel prefix scan produces the same
states as the sequential recurrence. In the specialized path,
\(a_{m,t}=p_m\) is static, and the Triton kernel evaluates the prefixes with
\texttt{tl.associative\_scan}. Each bank therefore retains \(O(T)\) work
with \(O(\log T)\) parallel scan depth.

This is a parallel affine scan rather than Mamba's selective scan: the
specialized transition is static, and the direct and cascaded banks remain
serially composed because the latter depends on the synthesized direct
trajectory. The scan also accumulates the seven per-mode statistics: log
energy and the real and imaginary radial-log moments at lags
\(\{1,2,4\}\). Parallel reassociation changes FP32 rounding order, so the
specialized path is numerical-parity gated rather than required to be
bitwise identical to the sequential reference.

\paragraph{Fused inference dataflow.}
The supported inference path is
\[
\begin{aligned}
\text{edge lift}
&\longrightarrow \text{writer scan + moments}\\
&\longrightarrow \text{synthesis + reader drive}\\
&\longrightarrow \text{reader moments-only scan}
\longrightarrow \text{affine logits}.
\end{aligned}
\]
The writer scan retains its modal trajectory because synthesis through
\(A_1^\top\) is required. A fused producer constructs the reader drive from
the synthesized direct trajectory, the residual stream, and the reader lift.
The reader then uses a moments-only parallel scan that does not materialize
or export its full state trajectory, returning the \(7M\) radial-log
descriptor directly. A final Triton kernel applies the affine classifier to
the two \(7M\) descriptors without a separate concatenation or linear launch.
These fusions preserve the main-paper forward map while reducing intermediate
global-memory traffic and kernel launches.

\paragraph{Complete-step training path.}
Training uses differentiable time-parallel scans for the forward states and
reverse adjoints, together with fused edge and terminal-reader operations.
The model and loss are compiled, and gradient-norm reduction, clipping, and
AdamW are combined in a fused optimizer tail. A CUDA~12.8 graph captures the
complete training step, including forward/loss, backward, optimization, and
post-update frame constraints. The inference graph similarly captures one
supported static-shape forward pass. Compilation, graph construction, and
warmup are excluded from the reported steady-state measurements.

\paragraph{Systems audit protocol.}
After freezing the specialization, the runtime audit compares
repository-native CNN1D, TCN, Transformer, S4D, S5, LRU, GRU, and LSTM with
official \texttt{mamba-ssm} 2.3.2.post1 using
\texttt{causal-conv1d} 1.6.1 and its fast path. The audit environment uses
PyTorch 2.9.1+cu128, Triton 3.5.1, Python 3.13.14, CUDA~12.8, and a single
24\,GiB RTX~4090.

The runtime audit uses two capacity controls, distinct from the task-level
18-candidate selection protocol. The \emph{approximately parameter-matched}
control selects, for each family, the nearest realizable variant from its
predeclared structural choices to the parameter count of the reference
\model{} design. Here \(D\) denotes the width of the reference design; a
matched baseline may therefore use a different native width. No inactive
parameters are added to force exact equality. The resulting 81 comparisons
cover nine one-factor shapes:
\[
\begin{aligned}
B&\in\{1,8,32,64\},       &&(T,D)=(512,64),\\
T&\in\{128,512,2048\},    &&(B,D)=(32,64),\\
D&\in\{16,32,64,128\},    &&(B,T)=(32,512).
\end{aligned}
\]
Inference uses the specialized kernel in the four \(D=64\) rows with
\(B\in\{32,64\}\), giving 36 baseline comparisons; the remaining five rows
use the reference graph, giving 45 comparisons. Complete-step measurements
use the differentiable training runtime at every shape. The maximum
parameter-count gap is 5.3\%, and all other gaps are below 2.5\%. This is an
approximate parameter control rather than a FLOP-, activation-, or
memory-matched comparison. The reported speedups should therefore be
interpreted as empirical results under the stated parameter and workload
controls.

The separate \emph{same-width} control fixes the external hidden width at
\(D=64\) and the workload at \((B,T)=(32,512)\), while retaining each
family's native internal structure. Its nine comparisons intentionally do
not match parameter counts and are reported separately from the
parameter-matched aggregate.
\begin{table*}[t]
\centering
\small
\setlength{\tabcolsep}{1mm}
\begin{tabular*}{\textwidth}{@{\extracolsep{\fill}}lrrrrrrrrr@{}}
\toprule
Design \(B/T/D\) & CNN & TCN & Trf. & Mamba & S4D & S5 & LRU & GRU & LSTM\\
\midrule
Param. 1/512/64 & 1.05/0.41 & 0.99/0.34 & 18.47/25.79 & 1.88/0.64 & 5.88/1.34 & 5.45/1.29 & 3.90/1.00 & 21.00/18.81 & 11.14/2.02\\
Param. 8/512/64 & 0.21/0.58 & 0.21/0.51 & 3.97/23.42 & 0.41/0.62 & 1.17/1.36 & 1.20/1.37 & 0.85/0.98 & 4.28/16.24 & 2.06/1.81\\
Param. 32/128/64 & 0.83/0.55 & 0.96/0.44 & 14.36/25.06 & 1.30/0.61 & 4.26/1.33 & 2.77/0.96 & 2.19/0.79 & 6.72/16.95 & 2.99/0.90\\
Param. 32/512/16 & 0.27/0.54 & 0.53/0.73 & 10.83/19.75 & 0.62/0.64 & 1.08/0.83 & 1.68/1.17 & 1.13/0.93 & 6.19/16.37 & 5.03/2.74\\
Param. 32/512/32 & 0.39/0.65 & 0.40/0.50 & 8.96/14.93 & 0.59/0.58 & 1.54/1.10 & 2.07/1.43 & 2.37/1.51 & 7.26/16.88 & 3.89/1.98\\
Param. 32/512/64 & 0.50/0.60 & 0.60/0.48 & 14.13/18.29 & 1.07/0.77 & 2.58/1.38 & 2.66/1.46 & 1.78/0.95 & 6.63/3.32 & 3.64/1.54\\
Param. 32/512/128 & 0.41/0.44 & 0.46/0.39 & 4.34/6.69 & 0.83/0.66 & 1.12/0.89 & 3.78/2.99 & 2.57/1.80 & 2.23/5.50 & 3.04/7.72\\
Param. 32/2048/64 & 0.34/0.48 & 0.38/0.30 & 44.38/20.20 & 1.09/0.92 & 1.72/1.42 & 3.17/2.34 & 1.62/1.30 & 7.11/5.48 & 3.84/2.47\\
Param. 64/512/64 & 0.42/0.48 & 0.48/0.39 & 16.67/12.01 & 1.02/0.85 & 1.99/1.40 & 2.20/1.41 & 1.39/0.88 & 4.17/8.38 & 2.39/1.20\\
\midrule
Same width 32/512/64 & 1.68/1.59 & 2.42/1.89 & 16.21/17.47 & 2.90/1.84 & 3.27/1.84 & 4.09/2.37 & 4.08/2.30 & 14.87/16.09 & 12.99/14.35\\
\bottomrule
\end{tabular*}
\caption{Complete 90-cell systems ledger. Each entry is
inference/complete-step speedup (baseline latency divided by ALPHABET
latency); values above one favor ALPHABET. ``Param.'' denotes approximately
parameter-matched comparisons (maximum gap 5.3\%; all but one at most 2.49\%).
The final row fixes \(D=64\) at \((B,T)=(32,512)\) without matching parameter
counts. Each cell is the median of seven paired CUDA-event groups.}
\label{tab:full90-shape-ledger}
\end{table*}

\paragraph{Measurement protocol.}
Correctness-valid runtime candidates are frozen by median latency at the
off-grid \((4,256,64)\) calibration shape. TF32 and autocast are disabled;
a repeat calibration targets 25\,ms and is excluded. Five warmups precede seven
seeded balanced ABBA/BAAB groups. Inference is forward-only; a complete step
includes cross entropy, backward, global gradient clipping, AdamW, and model
constraints. Each family uses its fastest correctness-valid calibrated path, and the
manifest records authoritative imports.

\paragraph{Parity and retained evidence.}
Maximum logit error is \(8.345\times10^{-7}\). Input/parameter gradients and
three-step optimization trajectories pass at both anchors for seeds
\(\{7,11,19\}\), with additional gates for \(T=2048\), \(T=31\), seven modes,
zero drive, finite differences, and NaN/Inf propagation. The fail-closed
140-group evaluator reports \texttt{status=pass}. Table entries are medians over
seven paired groups; grid summaries are computed from the 90 cell medians
(180 phase rows and 1,260 timing groups).
\section{Evaluation Protocol}
\label{app:maximal-fairness}
\label{app:protocol}

The nine control families are CNN1D, TCN, Transformer, Mamba, S4D, S5,
LRU, GRU, and LSTM. Candidate counts and final-seed counts are matched across
families: every task--family cell uses 18 candidates and five final-training
seeds. Native parameter counts, FLOPs, optimizer steps, and wall-clock costs
are not matched. TCN denotes the disclosed causal dilated CNN control rather
than the canonical two-convolution residual TCN. The capacity schedule for the
proposed \model{} is specified separately from the control-family ledger.

\paragraph{Proposed-model capacity schedule.}
For the complete benchmark, \model{} uses the six capacity pairs
\[
(D,M)\in
\{(32,8),(32,16),(64,16),(64,32),(128,16),(128,32)\},
\]
each combined with the three common optimizer recipes. This gives 18
candidates per task--model cell. The schedule is part of the proposed model's
benchmark specification and is not treated as an additional control-family
row.

\subsection{Model configuration ledger}
\label{app:experiment-ledger}

Table~\ref{tab:model-config-ledger} lists the structural configurations
available to each control family. The common candidate-selection,
optimization, and TEST-sealing rules are given below; diagnostic-specific
budgets are stated with their corresponding protocols.

\begin{table}[t]
\centering
\small
\setlength{\tabcolsep}{1mm}
\begin{tabular*}{\columnwidth}{@{\extracolsep{\fill}}lll@{}}
\toprule
Family group & Component & Choices \\
\midrule
CNN1D & depth/kernel & \((2,3)\), \((4,5)\) \\
TCN & depth/kernel & \((3,3)\), \((5,5)\) \\
Transformer & depth/heads & \((1,2)\), \((2,4)\) \\
Mamba & state/conv & \((16,3)\), \((32,4)\) \\
S4D & depth/state & \((1,16)\), \((3,16)\) \\
S5/LRU/GRU/LSTM & depth/state & \((1,16)\), \((2,32)\) \\
\bottomrule
\end{tabular*}
\caption{Control-family configuration ledger. Each family uses widths
\(\{32,64,128\}\), the two listed structural choices, and recipes A/B/C,
giving 18 candidates per task--family cell. Recipes A/B/C respectively use
the learning-rate and clipping pairs
\((10^{-3},0.5)\), \((3{\times}10^{-3},1)\), and
\((10^{-2},2)\), with weight decay \(10^{-4}\) and effective batch size 64.
Width 128 uses microbatches of 32 with two-step gradient accumulation; the
other widths use batches of 64.}
\label{tab:model-config-ledger}
\end{table}

Stage~1 evaluates all 18 candidates at seed 7 and retains the validation
Top-6. Stage~2 evaluates those six candidates at seeds 11 and 19 and freezes
the configuration with the highest mean validation score over seeds
\(\{7,11,19\}\). Exact ties are resolved using the predeclared configuration
key. Regular-sequence tasks use at most 100 epochs. Final training uses seeds
\(\{23,31,43,47,59\}\), and official TEST data remain sealed until
configuration selection is complete.

\subsection{Tasks, objectives, and data sealing}

The fixed registry contains \BenchmarkTaskCount{} independently trained tasks:
\BenchmarkUCRTaskCount{} UCR datasets \cite{dau2019ucr},
\BenchmarkGeneralTaskCount{} general sequence, ECG, clinical/activity, and
forecasting tasks, and \BenchmarkUEAFaultTaskCount{} multivariate and
vibration/fault tasks, including UEA \cite{bagnall2018uea}. Registry
membership was determined using public data availability, reproducible split
and preprocessing requirements, compatibility with the common sequence
interface, and broad coverage across task and application types. The registry
is neither exhaustive nor intended to be statistically representative of all
time-series problems.

Non-UCR tasks use public releases and frozen manifests: PTB-XL and MIT-BIH ECG
\cite{wagner2020ptbxl,moody2001mitbih,dechazal2004heartbeat},
CWRU bearing signals \cite{smith2015cwru}, ETT, Electricity, and Weather
\cite{zhou2021informer,trindade2015electricity,mpibgcweather}, Sequential
CIFAR-10 \cite{krizhevsky2009cifar}, and balanced AudioSet VGGish features
\cite{gemmeke2017audioset,hershey2017audio,audiosetfeatures}. The remaining
registry sources are PhysioNet 2012/2019
\cite{silva2012physionet,reyna2020sepsis}, PAMAP2
\cite{reiss2012pamap2,reiss2012pamap2dataset}, UCI Localization Human Activity
\cite{uci2008localization}, USHCN as processed by GRU-ODE-Bayes
\cite{debrouwer2019gruode}, ISRUC-Sleep \cite{khalighi2016isruc},
Chapman-Shaoxing ECG \cite{zheng2020chapman}, and MFPT and Paderborn bearing
data \cite{mfptdataset,lessmeier2016paderborn}. No dataset was added or removed
because of observed \model{} or baseline performance. Released manifests fix
the task-specific splits and preprocessing procedures.

UCR and UEA configuration selection uses only official TRAIN data; official
TEST data remain sealed until final evaluation. Forecasting windows do not
cross chronological TRAIN/validation/TEST boundaries. Other external tasks
retain their predefined record-, group-, or patient-level splits. Selection
artifacts contain no TEST tensors or TEST-set identifiers. Normalization and
all other data-dependent preprocessing are fitted without access to TEST data.
\section{Diagnostic Protocols}
\label{app:finite-bank-margin}

\paragraph{Common protocol.}

These TRAIN-only diagnostics test whether a finite bank separates temporal
dependence when static first and second moments agree. UCR experiments use a
stratified 80/20 official-TRAIN split, fit preprocessing on the optimization
fold, and never load official TEST. Unless stated otherwise, we use
\((D,M)=(64,16)\), 100 epochs, seeds \(\{7,11,19\}\), and confirmatory
trial~4: AdamW with learning rate \(3{\times}10^{-3}\), weight decay
\(10^{-4}\), effective batch size 64, and gradient clipping at 1. The
synthetic control uses 60 epochs.

\subsection{Moment-matched spectral control}

Class 0 is \(X_t=Z_t\), and class 1 is
\[
X_t=aZ_t+bZ_{t-5},\qquad
a^2+b^2=1,\qquad 2ab=\epsilon,
\]
for independent standard Gaussian innovations. Both are stationary,
Gaussian, and unit variance; their population autocovariances agree for
\(|k|\leq4\) and first differ at lag 5.

We sweep \(\epsilon\in\{.1,.2,.4,.8\}\) over seeds
23, 31, 43, 47, and 59, using 512 optimization and 256 validation paths of
length 128 per seed. Four estimators share the same realized samples: the
exact Gaussian Bayes likelihood ratio, a nearest-prototype classifier using
only empirical raw \(\widehat\Gamma(0{:}4)\), the capacity-matched energy-only
model, and the full radial-log model. The neural models use \((D,M)=(64,16)\)
and 60 epochs. The main-paper diagnostic reports the resulting validation
curves; this appendix records the protocol and TEST exclusion only.

\subsection{Capacity-matched structural ablation}

The energy-only control uses the same 17 datasets, seeds, two-bank backbone,
optimization recipe, TRAIN-derived 80/20 split, and selection rule as the
other ablations. It gives a nonlinear MLP only the direct and cascaded
lag-zero energies and matches the complete model's trainable count within
\(0.2\%\). The two-scan comparison gives a one-scan baseline the same 224 real
coordinates with \(M'=32\) and equalizes trainable parameters with \(D'=72\).
The main-paper ablation reports the resulting task-paired comparison.

\subsection{Pole-coordinate audit}

We audit learned- and fixed-random-pole checkpoints from the same 17-task
campaign at seeds \(\{23,31,43,47,59\}\). Using the affine margin identity,
we rank the 32 bank--pole blocks by optimization-fold mean absolute
contribution. For each checkpoint, neutralization replaces a selected
seven-coordinate block by its optimization-fold mean, while retention keeps
only selected blocks. Top and bottom subsets use the frozen ranking; random
entries average 64 uniformly sampled subsets per checkpoint. The main-paper
attribution audit reports the resulting validation contrasts; no official TEST
data are used here.
\subsection{Synthetic-corruption protocol}
\label{app:boundary-results}

We train \model{} and nine controls on 17 UCR tasks and seeds
\(\{23,31,43,47,59\}\), using configurations frozen from TRAIN-derived
validation. The \(10\times17\times5=850\) checkpoints are evaluated without
retraining under Gaussian noise \(\sigma\in\{.1,.2\}\), amplitude multipliers
\(\{.5,1.5\}\), and half-rate downsampling followed by restoration to the
original length. The main paper reports the corruption outcomes; this section
records the evaluation contract and TEST-free configuration freeze.
\section{Complete TEST Results}
\label{app:complete-test-results}


\begin{table*}[t]
\centering
\small
\setlength{\tabcolsep}{1mm}
\renewcommand{\arraystretch}{0.94}
\begin{tabular}{@{}lrrrrr@{}}
\toprule
Dataset & ALPHABET & CNN1D & TCN & Mamba & S4D\\
\midrule
ACSF1 (Bal. Acc. $\uparrow$) & \textbf{0.788 (0.020)} & 0.718 (0.059) & 0.682 (0.029) & 0.750 (0.019) & 0.680 (0.031)\\
Adiac (Bal. Acc. $\uparrow$) & 0.536 (0.148) & 0.407 (0.118) & 0.455 (0.017) & 0.492 (0.064) & \textbf{0.638 (0.061)}\\
ArrowHead (Bal. Acc. $\uparrow$) & \textbf{0.630 (0.040)} & 0.543 (0.132) & \underline{0.556 (0.029)} & 0.474 (0.078) & 0.501 (0.046)\\
BME (Bal. Acc. $\uparrow$) & \textbf{0.995 (0.006)} & 0.752 (0.020) & 0.815 (0.024) & 0.932 (0.057) & 0.873 (0.063)\\
CinCECGTorso (Bal. Acc. $\uparrow$) & \underline{0.674 (0.052)} & 0.638 (0.138) & 0.309 (0.052) & 0.557 (0.045) & 0.326 (0.027)\\
Coffee (Bal. Acc. $\uparrow$) & \textbf{1.000 (0.000)} & 0.865 (0.162) & \textbf{1.000 (0.000)} & \underline{0.967 (0.000)} & 0.887 (0.098)\\
CricketX (Bal. Acc. $\uparrow$) & 0.683 (0.018) & \underline{0.739 (0.042)} & 0.673 (0.041) & 0.676 (0.014) & \textbf{0.747 (0.009)}\\
Crop (Bal. Acc. $\uparrow$) & 0.726 (0.006) & \underline{0.742 (0.004)} & 0.733 (0.003) & 0.706 (0.004) & 0.738 (0.003)\\
ECG200 (Bal. Acc. $\uparrow$) & 0.831 (0.037) & \underline{0.832 (0.029)} & \textbf{0.845 (0.028)} & 0.776 (0.039) & 0.811 (0.025)\\
ECG5000 (Bal. Acc. $\uparrow$) & 0.534 (0.010) & 0.518 (0.026) & 0.524 (0.006) & 0.513 (0.024) & 0.529 (0.015)\\
ECGFiveDays (Bal. Acc. $\uparrow$) & 0.830 (0.122) & 0.818 (0.025) & \textbf{0.968 (0.036)} & 0.818 (0.040) & 0.737 (0.010)\\
EOGHorizontalSignal (Bal. Acc. $\uparrow$) & 0.560 (0.030) & 0.512 (0.055) & 0.469 (0.043) & 0.492 (0.040) & \textbf{0.599 (0.027)}\\
Earthquakes (Bal. Acc. $\uparrow$) & 0.526 (0.035) & 0.500 (0.000) & 0.500 (0.000) & 0.500 (0.000) & 0.500 (0.000)\\
FordA (Bal. Acc. $\uparrow$) & \textbf{0.953 (0.005)} & 0.930 (0.018) & \underline{0.947 (0.009)} & 0.916 (0.009) & 0.945 (0.003)\\
FordB (Bal. Acc. $\uparrow$) & \textbf{0.823 (0.017)} & 0.804 (0.010) & \underline{0.821 (0.034)} & 0.784 (0.030) & 0.803 (0.011)\\
GunPoint (Bal. Acc. $\uparrow$) & \underline{0.982 (0.010)} & \textbf{0.999 (0.003)} & 0.954 (0.015) & 0.927 (0.042) & 0.966 (0.007)\\
InsectEPGRegularTrain (Bal. Acc. $\uparrow$) & \textbf{1.000 (0.000)} & \textbf{1.000 (0.000)} & \textbf{1.000 (0.000)} & \textbf{1.000 (0.000)} & \textbf{1.000 (0.000)}\\
InsectWingbeatSound (Bal. Acc. $\uparrow$) & 0.511 (0.014) & 0.367 (0.033) & 0.566 (0.008) & 0.439 (0.016) & \textbf{0.611 (0.013)}\\
ItalyPowerDemand (Bal. Acc. $\uparrow$) & \textbf{0.944 (0.018)} & 0.882 (0.024) & 0.761 (0.139) & 0.672 (0.147) & 0.917 (0.020)\\
MoteStrain (Bal. Acc. $\uparrow$) & 0.879 (0.026) & 0.825 (0.043) & 0.828 (0.025) & 0.849 (0.015) & 0.872 (0.010)\\
Plane (Bal. Acc. $\uparrow$) & \underline{0.996 (0.005)} & \textbf{1.000 (0.000)} & 0.944 (0.037) & 0.995 (0.007) & 0.959 (0.013)\\
PowerCons (Bal. Acc. $\uparrow$) & \underline{0.983 (0.012)} & 0.830 (0.056) & 0.978 (0.004) & 0.963 (0.009) & 0.979 (0.014)\\
Rock (Bal. Acc. $\uparrow$) & 0.321 (0.058) & \underline{0.440 (0.019)} & 0.360 (0.092) & 0.297 (0.037) & 0.271 (0.071)\\
ShapesAll (Bal. Acc. $\uparrow$) & 0.781 (0.008) & 0.769 (0.059) & 0.712 (0.040) & 0.769 (0.015) & \textbf{0.823 (0.016)}\\
StarLightCurves (Bal. Acc. $\uparrow$) & 0.947 (0.015) & 0.922 (0.031) & 0.946 (0.010) & 0.938 (0.008) & \underline{0.951 (0.017)}\\
SwedishLeaf (Bal. Acc. $\uparrow$) & 0.897 (0.011) & \underline{0.931 (0.059)} & 0.894 (0.008) & 0.914 (0.009) & \textbf{0.936 (0.008)}\\
Trace (Bal. Acc. $\uparrow$) & \textbf{1.000 (0.000)} & 0.722 (0.283) & 0.974 (0.009) & \underline{0.997 (0.006)} & 0.980 (0.010)\\
TwoLeadECG (Bal. Acc. $\uparrow$) & 0.853 (0.144) & \textbf{0.988 (0.021)} & 0.797 (0.077) & 0.897 (0.056) & \underline{0.963 (0.013)}\\
Wafer (Bal. Acc. $\uparrow$) & \textbf{0.987 (0.005)} & 0.940 (0.056) & \underline{0.986 (0.007)} & 0.952 (0.023) & 0.985 (0.005)\\
Worms (Bal. Acc. $\uparrow$) & 0.569 (0.097) & 0.577 (0.066) & 0.525 (0.114) & 0.554 (0.082) & \textbf{0.628 (0.046)}\\
\bottomrule
\end{tabular}
\caption{Panel (a): UCR classification, comparison block 1. Entries are mean (sample SD) over final-seed replicates; best and second-best displayed means are bold and underlined.}
\label{tab:posthoc-pointwise-results}
\label{tab:posthoc-pointwise-ucr}
\end{table*}
\begin{table*}[t]
\ContinuedFloat
\centering
\small
\setlength{\tabcolsep}{1mm}
\renewcommand{\arraystretch}{0.94}
\begin{tabular}{@{}lrrrrrr@{}}
\toprule
Dataset & ALPHABET & S5 & LRU & GRU & LSTM & Transformer\\
\midrule
ACSF1 (Bal. Acc. $\uparrow$) & \textbf{0.788 (0.020)} & 0.722 (0.102) & 0.680 (0.042) & 0.714 (0.050) & \underline{0.776 (0.023)} & 0.588 (0.081)\\
Adiac (Bal. Acc. $\uparrow$) & 0.536 (0.148) & 0.415 (0.062) & 0.276 (0.092) & 0.230 (0.036) & 0.275 (0.064) & \underline{0.563 (0.044)}\\
ArrowHead (Bal. Acc. $\uparrow$) & \textbf{0.630 (0.040)} & 0.392 (0.087) & 0.534 (0.066) & 0.524 (0.040) & 0.483 (0.065) & 0.494 (0.050)\\
BME (Bal. Acc. $\uparrow$) & \textbf{0.995 (0.006)} & 0.913 (0.055) & 0.884 (0.041) & \underline{0.983 (0.039)} & \textbf{0.995 (0.006)} & 0.907 (0.101)\\
CinCECGTorso (Bal. Acc. $\uparrow$) & \underline{0.674 (0.052)} & 0.262 (0.014) & 0.546 (0.144) & 0.571 (0.044) & 0.352 (0.017) & \textbf{0.837 (0.040)}\\
Coffee (Bal. Acc. $\uparrow$) & \textbf{1.000 (0.000)} & 0.824 (0.147) & 0.933 (0.061) & 0.960 (0.015) & 0.836 (0.213) & \underline{0.967 (0.024)}\\
CricketX (Bal. Acc. $\uparrow$) & 0.683 (0.018) & 0.732 (0.017) & 0.605 (0.047) & 0.674 (0.023) & 0.690 (0.021) & 0.539 (0.028)\\
Crop (Bal. Acc. $\uparrow$) & 0.726 (0.006) & 0.709 (0.007) & \textbf{0.747 (0.014)} & 0.717 (0.008) & 0.709 (0.008) & 0.713 (0.008)\\
ECG200 (Bal. Acc. $\uparrow$) & 0.831 (0.037) & 0.813 (0.024) & 0.809 (0.044) & 0.751 (0.020) & 0.803 (0.041) & 0.822 (0.034)\\
ECG5000 (Bal. Acc. $\uparrow$) & 0.534 (0.010) & \underline{0.537 (0.015)} & 0.477 (0.043) & 0.530 (0.012) & 0.502 (0.039) & \textbf{0.550 (0.015)}\\
ECGFiveDays (Bal. Acc. $\uparrow$) & 0.830 (0.122) & 0.715 (0.076) & \underline{0.911 (0.085)} & 0.733 (0.011) & 0.677 (0.070) & 0.770 (0.048)\\
EOGHorizontalSignal (Bal. Acc. $\uparrow$) & 0.560 (0.030) & 0.542 (0.026) & 0.405 (0.053) & \underline{0.594 (0.026)} & 0.520 (0.028) & 0.509 (0.032)\\
Earthquakes (Bal. Acc. $\uparrow$) & 0.526 (0.035) & 0.500 (0.000) & 0.500 (0.000) & \textbf{0.548 (0.046)} & 0.500 (0.000) & \underline{0.544 (0.067)}\\
FordA (Bal. Acc. $\uparrow$) & \textbf{0.953 (0.005)} & 0.941 (0.011) & 0.935 (0.011) & 0.937 (0.007) & 0.918 (0.011) & 0.561 (0.022)\\
FordB (Bal. Acc. $\uparrow$) & \textbf{0.823 (0.017)} & 0.816 (0.017) & 0.804 (0.011) & 0.808 (0.019) & 0.780 (0.037) & 0.577 (0.026)\\
GunPoint (Bal. Acc. $\uparrow$) & \underline{0.982 (0.010)} & 0.924 (0.048) & 0.900 (0.097) & 0.952 (0.022) & 0.817 (0.130) & 0.929 (0.034)\\
InsectEPGRegularTrain (Bal. Acc. $\uparrow$) & \textbf{1.000 (0.000)} & \textbf{1.000 (0.000)} & \textbf{1.000 (0.000)} & \textbf{1.000 (0.000)} & \textbf{1.000 (0.000)} & \textbf{1.000 (0.000)}\\
InsectWingbeatSound (Bal. Acc. $\uparrow$) & 0.511 (0.014) & \underline{0.581 (0.008)} & 0.528 (0.059) & 0.503 (0.026) & 0.491 (0.026) & 0.499 (0.043)\\
ItalyPowerDemand (Bal. Acc. $\uparrow$) & \textbf{0.944 (0.018)} & 0.882 (0.059) & \underline{0.924 (0.041)} & 0.755 (0.165) & 0.869 (0.052) & \textbf{0.944 (0.028)}\\
MoteStrain (Bal. Acc. $\uparrow$) & 0.879 (0.026) & \textbf{0.897 (0.008)} & \underline{0.884 (0.024)} & 0.880 (0.024) & 0.839 (0.047) & 0.800 (0.021)\\
Plane (Bal. Acc. $\uparrow$) & \underline{0.996 (0.005)} & 0.960 (0.024) & 0.971 (0.015) & 0.974 (0.015) & 0.971 (0.035) & 0.940 (0.078)\\
PowerCons (Bal. Acc. $\uparrow$) & \underline{0.983 (0.012)} & 0.977 (0.005) & 0.889 (0.052) & 0.941 (0.010) & 0.936 (0.009) & \textbf{1.000 (0.000)}\\
Rock (Bal. Acc. $\uparrow$) & 0.321 (0.058) & 0.346 (0.060) & 0.275 (0.046) & 0.373 (0.041) & 0.368 (0.072) & \textbf{0.466 (0.006)}\\
ShapesAll (Bal. Acc. $\uparrow$) & 0.781 (0.008) & \underline{0.792 (0.016)} & 0.443 (0.309) & 0.766 (0.024) & 0.705 (0.024) & 0.667 (0.008)\\
StarLightCurves (Bal. Acc. $\uparrow$) & 0.947 (0.015) & \underline{0.951 (0.006)} & 0.935 (0.027) & \textbf{0.953 (0.008)} & 0.927 (0.015) & 0.900 (0.040)\\
SwedishLeaf (Bal. Acc. $\uparrow$) & 0.897 (0.011) & 0.905 (0.033) & 0.752 (0.184) & 0.863 (0.019) & 0.842 (0.021) & 0.828 (0.025)\\
Trace (Bal. Acc. $\uparrow$) & \textbf{1.000 (0.000)} & 0.935 (0.109) & 0.978 (0.018) & 0.957 (0.009) & 0.992 (0.013) & 0.842 (0.146)\\
TwoLeadECG (Bal. Acc. $\uparrow$) & 0.853 (0.144) & 0.849 (0.098) & 0.926 (0.102) & 0.961 (0.021) & 0.930 (0.063) & 0.667 (0.058)\\
Wafer (Bal. Acc. $\uparrow$) & \textbf{0.987 (0.005)} & 0.976 (0.006) & 0.982 (0.017) & 0.963 (0.006) & 0.951 (0.025) & 0.976 (0.012)\\
Worms (Bal. Acc. $\uparrow$) & 0.569 (0.097) & 0.554 (0.084) & 0.538 (0.085) & \underline{0.609 (0.082)} & 0.504 (0.044) & 0.377 (0.041)\\
\bottomrule
\end{tabular}
\caption{Panel (a), continued: UCR classification, comparison block 2. ALPHABET is repeated to keep both comparison blocks directly readable.}
\label{tab:posthoc-pointwise-results-continued}
\end{table*}
\begin{table*}[t]
\centering
\small
\setlength{\tabcolsep}{1mm}
\renewcommand{\arraystretch}{0.94}
\begin{tabular}{@{}lrrrrr@{}}
\toprule
Dataset & ALPHABET & CNN1D & TCN & Mamba & S4D\\
\midrule
PTB-XL (AUROC $\uparrow$) & \underline{0.921 (0.002)} & \underline{0.921 (0.003)} & \textbf{0.925 (0.002)} & 0.907 (0.002) & 0.913 (0.003)\\
\quad\emph{Selection: AUPRC $\uparrow$} & 0.810 (0.005) & 0.809 (0.005) & 0.817 (0.005) & 0.781 (0.004) & 0.796 (0.006)\\
MIT-BIH (Acc. $\uparrow$) & \underline{0.873 (0.030)} & \textbf{0.886 (0.009)} & 0.817 (0.065) & 0.871 (0.025) & 0.863 (0.048)\\
\quad\emph{Selection: Bal. Acc. $\uparrow$} & 0.337 (0.017) & 0.268 (0.034) & 0.442 (0.063) & 0.358 (0.050) & 0.372 (0.039)\\
CWRU (Acc. $\uparrow$) & \textbf{0.978 (0.028)} & 0.974 (0.031) & 0.936 (0.038) & 0.974 (0.015) & 0.941 (0.023)\\
\quad\emph{Selection: Bal. Acc. $\uparrow$} & 0.974 (0.033) & 0.969 (0.036) & 0.925 (0.046) & 0.970 (0.018) & 0.940 (0.017)\\
ETTm1 (MSE $\downarrow$) & \underline{0.576 (0.023)} & 0.871 (0.095) & 0.780 (0.031) & 0.811 (0.067) & 0.620 (0.027)\\
ETTm2 (MSE $\downarrow$) & 0.618 (0.200) & 0.600 (0.095) & 0.542 (0.085) & 0.592 (0.215) & 0.596 (0.087)\\
Electricity (MSE $\downarrow$) & 0.299 (0.007) & 0.295 (0.006) & 0.312 (0.010) & 0.322 (0.005) & 0.336 (0.013)\\
Weather (MSE $\downarrow$) & \underline{0.196 (0.010)} & 0.272 (0.009) & 0.291 (0.020) & 0.235 (0.018) & 0.222 (0.007)\\
Sequential CIFAR (Acc. $\uparrow$) & 0.601 (0.007) & \textbf{0.666 (0.014)} & 0.633 (0.004) & 0.613 (0.007) & \underline{0.653 (0.008)}\\
\quad\emph{Selection: Bal. Acc. $\uparrow$} & 0.601 (0.007) & 0.666 (0.014) & 0.633 (0.004) & 0.613 (0.007) & 0.653 (0.008)\\
AudioSet (AUPRC $\uparrow$) & 0.114 (0.002) & 0.141 (0.001) & 0.118 (0.003) & 0.125 (0.001) & 0.141 (0.000)\\
PhysioNet-2012 (AUROC $\uparrow$) & 0.783 (0.014) & 0.773 (0.005) & 0.773 (0.016) & 0.769 (0.028) & 0.805 (0.003)\\
\quad\emph{Selection: AUPRC $\uparrow$} & 0.382 (0.028) & 0.394 (0.015) & 0.384 (0.019) & 0.376 (0.027) & 0.432 (0.009)\\
PhysioNet-2019 (AUROC $\uparrow$) & 0.889 (0.009) & 0.890 (0.008) & \underline{0.894 (0.009)} & 0.839 (0.018) & 0.890 (0.009)\\
\quad\emph{Selection: AUPRC $\uparrow$} & 0.604 (0.015) & 0.494 (0.018) & 0.506 (0.014) & 0.507 (0.012) & 0.639 (0.012)\\
PAM (PAMAP2) (Acc. $\uparrow$) & 0.969 (0.011) & \textbf{0.980 (0.001)} & \underline{0.979 (0.006)} & 0.977 (0.003) & 0.973 (0.002)\\
\quad\emph{Selection: Macro-F1 $\uparrow$} & 0.972 (0.010) & 0.983 (0.001) & 0.981 (0.006) & 0.979 (0.003) & 0.977 (0.002)\\
Human Activity (Bal. Acc. $\uparrow$) & 0.682 (0.016) & 0.612 (0.033) & 0.643 (0.013) & 0.646 (0.027) & \textbf{0.768 (0.018)}\\
USHCN-Daily (MSE $\downarrow$) & 0.375 (0.006) & 0.374 (0.005) & 0.379 (0.005) & 0.384 (0.012) & \underline{0.265 (0.084)}\\
ISRUC-Sleep (Bal. Acc. $\uparrow$) & 0.699 (0.006) & 0.557 (0.024) & 0.693 (0.006) & 0.655 (0.009) & \textbf{0.708 (0.013)}\\
Chapman-Shaoxing (Bal. Acc. $\uparrow$) & 0.422 (0.025) & 0.321 (0.003) & 0.384 (0.008) & 0.398 (0.009) & \underline{0.429 (0.004)}\\
ETTh1 (MSE $\downarrow$) & 1.041 (0.058) & \underline{1.005 (0.019)} & 1.695 (0.103) & 1.540 (0.129) & 1.200 (0.038)\\
ETTh2 (MSE $\downarrow$) & 2.515 (0.129) & 3.732 (0.089) & 3.668 (1.093) & 2.895 (0.297) & \textbf{1.715 (0.103)}\\
Traffic (MSE $\downarrow$) & \textbf{0.676 (0.010)} & 0.943 (0.029) & 0.827 (0.071) & 0.772 (0.036) & 0.851 (0.038)\\
ILI (MSE $\downarrow$) & 5.693 (0.764) & 8.861 (1.057) & 7.027 (0.655) & 6.529 (1.060) & \underline{5.096 (0.122)}\\
Exchange-Rate (MSE $\downarrow$) & 1.299 (0.429) & 0.976 (0.222) & 1.163 (0.176) & 1.949 (0.286) & \textbf{0.756 (0.140)}\\
\bottomrule
\end{tabular}
\caption{Panel (b): general sequence, ECG, irregular clinical/activity, and forecasting tasks, comparison block 1. Entries are mean (sample SD) over final-seed replicates; best and second-best displayed means are bold and underlined.}
\label{tab:posthoc-pointwise-external}
\end{table*}
\begin{table*}[t]
\ContinuedFloat
\centering
\small
\setlength{\tabcolsep}{1mm}
\renewcommand{\arraystretch}{0.94}
\begin{tabular}{@{}lrrrrrr@{}}
\toprule
Dataset & ALPHABET & S5 & LRU & GRU & LSTM & Transformer\\
\midrule
PTB-XL (AUROC $\uparrow$) & \underline{0.921 (0.002)} & 0.915 (0.002) & 0.914 (0.004) & 0.917 (0.003) & 0.908 (0.004) & 0.882 (0.003)\\
\quad\emph{Selection: AUPRC $\uparrow$} & 0.810 (0.005) & 0.798 (0.005) & 0.793 (0.012) & 0.801 (0.006) & 0.787 (0.009) & 0.737 (0.008)\\
MIT-BIH (Acc. $\uparrow$) & \underline{0.873 (0.030)} & 0.797 (0.089) & \textbf{0.886 (0.028)} & 0.848 (0.039) & 0.840 (0.026) & 0.816 (0.093)\\
\quad\emph{Selection: Bal. Acc. $\uparrow$} & 0.337 (0.017) & 0.368 (0.029) & 0.293 (0.032) & 0.372 (0.019) & 0.338 (0.023) & 0.385 (0.054)\\
CWRU (Acc. $\uparrow$) & \textbf{0.978 (0.028)} & 0.939 (0.013) & 0.906 (0.047) & \underline{0.975 (0.008)} & 0.964 (0.005) & 0.839 (0.012)\\
\quad\emph{Selection: Bal. Acc. $\uparrow$} & 0.974 (0.033) & 0.933 (0.011) & 0.902 (0.048) & 0.974 (0.010) & 0.962 (0.005) & 0.839 (0.022)\\
ETTm1 (MSE $\downarrow$) & \underline{0.576 (0.023)} & 0.634 (0.030) & 0.577 (0.027) & 0.691 (0.055) & 0.778 (0.035) & \textbf{0.555 (0.029)}\\
ETTm2 (MSE $\downarrow$) & 0.618 (0.200) & \underline{0.493 (0.107)} & 0.570 (0.220) & 0.525 (0.074) & 0.640 (0.069) & \textbf{0.400 (0.046)}\\
Electricity (MSE $\downarrow$) & 0.299 (0.007) & 0.377 (0.011) & 0.320 (0.016) & \underline{0.287 (0.010)} & 0.292 (0.007) & \textbf{0.280 (0.006)}\\
Weather (MSE $\downarrow$) & \underline{0.196 (0.010)} & 0.236 (0.015) & 0.245 (0.016) & 0.224 (0.009) & 0.275 (0.023) & \textbf{0.169 (0.007)}\\
Sequential CIFAR (Acc. $\uparrow$) & 0.601 (0.007) & 0.636 (0.007) & 0.620 (0.005) & 0.632 (0.007) & 0.617 (0.009) & 0.585 (0.003)\\
\quad\emph{Selection: Bal. Acc. $\uparrow$} & 0.601 (0.007) & 0.636 (0.007) & 0.620 (0.005) & 0.632 (0.007) & 0.617 (0.009) & 0.585 (0.003)\\
AudioSet (AUPRC $\uparrow$) & 0.114 (0.002) & \underline{0.142 (0.003)} & 0.133 (0.003) & 0.129 (0.001) & 0.134 (0.000) & \textbf{0.151 (0.001)}\\
PhysioNet-2012 (AUROC $\uparrow$) & 0.783 (0.014) & \textbf{0.822 (0.003)} & 0.769 (0.004) & \underline{0.815 (0.009)} & 0.803 (0.003) & 0.778 (0.009)\\
\quad\emph{Selection: AUPRC $\uparrow$} & 0.382 (0.028) & 0.463 (0.004) & 0.380 (0.017) & 0.439 (0.030) & 0.398 (0.012) & 0.375 (0.018)\\
PhysioNet-2019 (AUROC $\uparrow$) & 0.889 (0.009) & \textbf{0.911 (0.008)} & 0.877 (0.020) & 0.850 (0.031) & 0.890 (0.008) & 0.852 (0.018)\\
\quad\emph{Selection: AUPRC $\uparrow$} & 0.604 (0.015) & 0.646 (0.013) & 0.596 (0.013) & 0.539 (0.030) & 0.585 (0.020) & 0.599 (0.002)\\
PAM (PAMAP2) (Acc. $\uparrow$) & 0.969 (0.011) & 0.967 (0.003) & 0.953 (0.012) & 0.951 (0.011) & 0.970 (0.009) & 0.944 (0.012)\\
\quad\emph{Selection: Macro-F1 $\uparrow$} & 0.972 (0.010) & 0.970 (0.004) & 0.960 (0.009) & 0.956 (0.010) & 0.973 (0.008) & 0.952 (0.012)\\
Human Activity (Bal. Acc. $\uparrow$) & 0.682 (0.016) & 0.731 (0.019) & 0.652 (0.028) & 0.698 (0.021) & 0.753 (0.046) & \underline{0.758 (0.046)}\\
USHCN-Daily (MSE $\downarrow$) & 0.375 (0.006) & 0.372 (0.033) & 0.379 (0.009) & 0.277 (0.082) & \textbf{0.242 (0.041)} & 0.382 (0.009)\\
ISRUC-Sleep (Bal. Acc. $\uparrow$) & 0.699 (0.006) & 0.653 (0.010) & 0.686 (0.009) & \underline{0.701 (0.007)} & 0.651 (0.023) & 0.524 (0.017)\\
Chapman-Shaoxing (Bal. Acc. $\uparrow$) & 0.422 (0.025) & \textbf{0.434 (0.021)} & 0.356 (0.013) & 0.386 (0.013) & 0.379 (0.008) & 0.284 (0.015)\\
ETTh1 (MSE $\downarrow$) & 1.041 (0.058) & 1.120 (0.262) & 1.158 (0.139) & 1.099 (0.110) & 1.196 (0.002) & \textbf{0.916 (0.187)}\\
ETTh2 (MSE $\downarrow$) & 2.515 (0.129) & 1.831 (0.045) & 2.189 (0.221) & 2.270 (0.219) & 2.423 (0.342) & \underline{1.720 (0.222)}\\
Traffic (MSE $\downarrow$) & \textbf{0.676 (0.010)} & 0.945 (0.096) & 0.846 (0.056) & 0.755 (0.003) & \underline{0.713 (0.016)} & 0.757 (0.023)\\
ILI (MSE $\downarrow$) & 5.693 (0.764) & 5.242 (0.302) & \textbf{5.028 (0.395)} & 5.831 (0.227) & 6.463 (0.237) & 5.114 (0.164)\\
Exchange-Rate (MSE $\downarrow$) & 1.299 (0.429) & 1.044 (0.297) & 1.156 (0.148) & 1.062 (0.161) & 1.096 (0.066) & \underline{0.899 (0.127)}\\
\bottomrule
\end{tabular}
\caption{Panel (b), continued: general sequence, ECG, irregular clinical/activity, and forecasting tasks, comparison block 2. ALPHABET is repeated to keep both comparison blocks directly readable.}
\label{tab:posthoc-pointwise-external-continued}
\end{table*}
\begin{table*}[t]
\centering
\small
\setlength{\tabcolsep}{1mm}
\renewcommand{\arraystretch}{0.94}
\begin{tabular}{@{}lrrrrr@{}}
\toprule
Dataset & ALPHABET & CNN1D & TCN & Mamba & S4D\\
\midrule
ArticularyWordRecognition (Bal. Acc. $\uparrow$) & 0.971 (0.008) & 0.956 (0.002) & 0.852 (0.036) & \textbf{0.978 (0.002)} & \underline{0.974 (0.005)}\\
AtrialFibrillation (Bal. Acc. $\uparrow$) & 0.267 (0.115) & 0.244 (0.077) & 0.289 (0.038) & 0.289 (0.038) & \textbf{0.400 (0.000)}\\
BasicMotions (Bal. Acc. $\uparrow$) & \textbf{1.000 (0.000)} & 0.967 (0.014) & \underline{0.992 (0.014)} & \textbf{1.000 (0.000)} & \textbf{1.000 (0.000)}\\
CharacterTrajectories (Bal. Acc. $\uparrow$) & 0.988 (0.005) & 0.950 (0.019) & 0.988 (0.008) & 0.987 (0.001) & \textbf{0.995 (0.001)}\\
Cricket (Bal. Acc. $\uparrow$) & \textbf{0.986 (0.000)} & \textbf{0.986 (0.000)} & 0.921 (0.008) & \textbf{0.986 (0.000)} & \underline{0.977 (0.008)}\\
DuckDuckGeese (Bal. Acc. $\uparrow$) & \underline{0.593 (0.101)} & 0.507 (0.095) & 0.300 (0.053) & 0.327 (0.095) & 0.447 (0.064)\\
EigenWorms (Bal. Acc. $\uparrow$) & \textbf{0.837 (0.027)} & 0.369 (0.055) & 0.582 (0.111) & \underline{0.793 (0.043)} & 0.755 (0.006)\\
Epilepsy (Bal. Acc. $\uparrow$) & \textbf{0.977 (0.001)} & \underline{0.962 (0.004)} & 0.939 (0.012) & 0.958 (0.000) & 0.945 (0.004)\\
ERing (Bal. Acc. $\uparrow$) & \textbf{0.862 (0.032)} & 0.806 (0.013) & 0.836 (0.028) & 0.851 (0.019) & \underline{0.859 (0.026)}\\
EthanolConcentration (Bal. Acc. $\uparrow$) & \textbf{0.289 (0.022)} & 0.271 (0.015) & 0.276 (0.017) & 0.247 (0.008) & 0.265 (0.013)\\
FaceDetection (Bal. Acc. $\uparrow$) & 0.630 (0.006) & 0.546 (0.005) & 0.550 (0.013) & 0.566 (0.008) & 0.611 (0.024)\\
FingerMovements (Bal. Acc. $\uparrow$) & \textbf{0.521 (0.013)} & 0.493 (0.028) & 0.470 (0.047) & \underline{0.507 (0.044)} & 0.474 (0.023)\\
HandMovementDirection (Bal. Acc. $\uparrow$) & 0.240 (0.014) & 0.278 (0.048) & 0.295 (0.063) & 0.286 (0.066) & \underline{0.352 (0.034)}\\
Handwriting (Bal. Acc. $\uparrow$) & 0.304 (0.021) & 0.288 (0.021) & 0.225 (0.065) & 0.310 (0.029) & 0.269 (0.008)\\
Heartbeat (Bal. Acc. $\uparrow$) & 0.635 (0.033) & 0.623 (0.030) & 0.631 (0.007) & 0.609 (0.051) & \textbf{0.660 (0.016)}\\
InsectWingbeat (Bal. Acc. $\uparrow$) & 0.602 (0.004) & 0.583 (0.005) & 0.552 (0.010) & 0.613 (0.007) & \textbf{0.640 (0.002)}\\
JapaneseVowels (Bal. Acc. $\uparrow$) & \textbf{0.984 (0.000)} & 0.982 (0.007) & 0.928 (0.009) & 0.977 (0.002) & \underline{0.983 (0.004)}\\
Libras (Bal. Acc. $\uparrow$) & \textbf{0.898 (0.026)} & \underline{0.887 (0.029)} & 0.617 (0.091) & 0.791 (0.059) & 0.831 (0.014)\\
LSST (Bal. Acc. $\uparrow$) & 0.442 (0.010) & 0.300 (0.023) & 0.482 (0.026) & \textbf{0.496 (0.013)} & \underline{0.491 (0.035)}\\
MotorImagery (Bal. Acc. $\uparrow$) & 0.547 (0.015) & 0.577 (0.015) & 0.497 (0.040) & 0.543 (0.025) & \textbf{0.593 (0.035)}\\
NATOPS (Bal. Acc. $\uparrow$) & \underline{0.950 (0.024)} & \textbf{0.952 (0.008)} & 0.863 (0.027) & 0.913 (0.050) & 0.889 (0.020)\\
PEMS-SF (Bal. Acc. $\uparrow$) & \underline{0.719 (0.004)} & 0.715 (0.045) & 0.702 (0.042) & 0.687 (0.029) & 0.683 (0.011)\\
PenDigits (Bal. Acc. $\uparrow$) & 0.978 (0.006) & \textbf{0.988 (0.003)} & 0.982 (0.002) & 0.978 (0.007) & 0.984 (0.003)\\
RacketSports (Bal. Acc. $\uparrow$) & 0.861 (0.011) & 0.841 (0.034) & 0.823 (0.020) & 0.848 (0.051) & \underline{0.918 (0.009)}\\
SelfRegulationSCP1 (Bal. Acc. $\uparrow$) & 0.731 (0.018) & 0.777 (0.019) & 0.772 (0.007) & 0.725 (0.041) & 0.769 (0.024)\\
SelfRegulationSCP2 (Bal. Acc. $\uparrow$) & 0.494 (0.024) & 0.463 (0.055) & 0.504 (0.029) & \underline{0.524 (0.022)} & 0.491 (0.018)\\
SpokenArabicDigits (Bal. Acc. $\uparrow$) & 0.987 (0.001) & 0.982 (0.003) & 0.981 (0.006) & 0.985 (0.003) & \textbf{0.992 (0.002)}\\
StandWalkJump (Bal. Acc. $\uparrow$) & \underline{0.333 (0.000)} & \underline{0.333 (0.000)} & 0.289 (0.077) & 0.267 (0.176) & \underline{0.333 (0.067)}\\
UWaveGestureLibrary (Bal. Acc. $\uparrow$) & 0.848 (0.015) & 0.801 (0.010) & 0.598 (0.010) & 0.845 (0.041) & \underline{0.861 (0.010)}\\
MFPT Bearing (Bal. Acc. $\uparrow$) & 0.733 (0.149) & 0.867 (0.098) & 0.781 (0.134) & 0.706 (0.088) & \textbf{0.984 (0.035)}\\
Paderborn (KAT) (Bal. Acc. $\uparrow$) & \textbf{0.612 (0.021)} & 0.419 (0.086) & 0.466 (0.051) & \underline{0.611 (0.048)} & 0.567 (0.122)\\
\bottomrule
\end{tabular}
\caption{Panel (c): UEA multivariate classification and vibration/fault diagnosis, comparison block 1. Entries are mean (sample SD) over final-seed replicates; best and second-best displayed means are bold and underlined.}
\label{tab:posthoc-pointwise-uea-fault}
\end{table*}
\begin{table*}[t]
\ContinuedFloat
\centering
\small
\setlength{\tabcolsep}{1mm}
\renewcommand{\arraystretch}{0.94}
\begin{tabular}{@{}lrrrrrr@{}}
\toprule
Dataset & ALPHABET & S5 & LRU & GRU & LSTM & Transformer\\
\midrule
ArticularyWordRecognition (Bal. Acc. $\uparrow$) & 0.971 (0.008) & 0.972 (0.008) & \textbf{0.978 (0.002)} & 0.972 (0.002) & 0.967 (0.003) & 0.923 (0.023)\\
AtrialFibrillation (Bal. Acc. $\uparrow$) & 0.267 (0.115) & \underline{0.378 (0.102)} & 0.356 (0.077) & \underline{0.378 (0.038)} & 0.222 (0.102) & 0.267 (0.115)\\
BasicMotions (Bal. Acc. $\uparrow$) & \textbf{1.000 (0.000)} & 0.975 (0.000) & \textbf{1.000 (0.000)} & 0.967 (0.014) & 0.958 (0.014) & 0.908 (0.038)\\
CharacterTrajectories (Bal. Acc. $\uparrow$) & 0.988 (0.005) & \textbf{0.995 (0.001)} & 0.982 (0.001) & 0.982 (0.002) & \underline{0.989 (0.004)} & 0.959 (0.016)\\
Cricket (Bal. Acc. $\uparrow$) & \textbf{0.986 (0.000)} & 0.968 (0.016) & 0.972 (0.000) & 0.954 (0.021) & 0.972 (0.014) & 0.903 (0.073)\\
DuckDuckGeese (Bal. Acc. $\uparrow$) & \underline{0.593 (0.101)} & \textbf{0.660 (0.020)} & 0.573 (0.117) & 0.487 (0.050) & 0.447 (0.042) & 0.507 (0.050)\\
EigenWorms (Bal. Acc. $\uparrow$) & \textbf{0.837 (0.027)} & 0.759 (0.070) & 0.652 (0.120) & 0.667 (0.181) & 0.747 (0.056) & 0.307 (0.186)\\
Epilepsy (Bal. Acc. $\uparrow$) & \textbf{0.977 (0.001)} & 0.936 (0.019) & 0.954 (0.017) & 0.959 (0.015) & 0.880 (0.018) & 0.875 (0.053)\\
ERing (Bal. Acc. $\uparrow$) & \textbf{0.862 (0.032)} & 0.856 (0.000) & 0.833 (0.020) & 0.767 (0.006) & 0.815 (0.011) & 0.819 (0.010)\\
EthanolConcentration (Bal. Acc. $\uparrow$) & \textbf{0.289 (0.022)} & 0.270 (0.011) & 0.276 (0.002) & 0.279 (0.012) & \underline{0.284 (0.029)} & \underline{0.284 (0.027)}\\
FaceDetection (Bal. Acc. $\uparrow$) & 0.630 (0.006) & \underline{0.639 (0.004)} & 0.607 (0.053) & 0.559 (0.010) & 0.570 (0.006) & \textbf{0.683 (0.014)}\\
FingerMovements (Bal. Acc. $\uparrow$) & \textbf{0.521 (0.013)} & 0.477 (0.045) & \textbf{0.521 (0.026)} & 0.492 (0.010) & 0.504 (0.026) & 0.479 (0.018)\\
HandMovementDirection (Bal. Acc. $\uparrow$) & 0.240 (0.014) & \textbf{0.353 (0.072)} & 0.267 (0.045) & 0.267 (0.057) & 0.277 (0.013) & 0.285 (0.034)\\
Handwriting (Bal. Acc. $\uparrow$) & 0.304 (0.021) & \underline{0.340 (0.018)} & 0.118 (0.042) & 0.309 (0.023) & \textbf{0.369 (0.026)} & 0.102 (0.008)\\
Heartbeat (Bal. Acc. $\uparrow$) & 0.635 (0.033) & \underline{0.650 (0.015)} & 0.649 (0.013) & 0.647 (0.017) & 0.635 (0.009) & 0.619 (0.022)\\
InsectWingbeat (Bal. Acc. $\uparrow$) & 0.602 (0.004) & \underline{0.633 (0.010)} & 0.592 (0.005) & 0.627 (0.001) & 0.629 (0.001) & 0.598 (0.003)\\
JapaneseVowels (Bal. Acc. $\uparrow$) & \textbf{0.984 (0.000)} & 0.978 (0.002) & 0.968 (0.003) & 0.971 (0.010) & 0.965 (0.004) & 0.978 (0.003)\\
Libras (Bal. Acc. $\uparrow$) & \textbf{0.898 (0.026)} & 0.778 (0.006) & 0.744 (0.019) & 0.772 (0.040) & 0.741 (0.037) & 0.706 (0.063)\\
LSST (Bal. Acc. $\uparrow$) & 0.442 (0.010) & 0.486 (0.018) & 0.474 (0.005) & 0.463 (0.017) & 0.474 (0.020) & 0.421 (0.025)\\
MotorImagery (Bal. Acc. $\uparrow$) & 0.547 (0.015) & 0.573 (0.042) & 0.577 (0.035) & \underline{0.590 (0.030)} & 0.557 (0.032) & 0.567 (0.070)\\
NATOPS (Bal. Acc. $\uparrow$) & \underline{0.950 (0.024)} & 0.859 (0.006) & 0.902 (0.031) & 0.874 (0.042) & 0.878 (0.044) & 0.828 (0.031)\\
PEMS-SF (Bal. Acc. $\uparrow$) & \underline{0.719 (0.004)} & 0.709 (0.028) & \textbf{0.725 (0.044)} & 0.690 (0.003) & 0.694 (0.020) & 0.688 (0.038)\\
PenDigits (Bal. Acc. $\uparrow$) & 0.978 (0.006) & 0.978 (0.005) & 0.985 (0.003) & \textbf{0.988 (0.003)} & \underline{0.987 (0.002)} & 0.971 (0.003)\\
RacketSports (Bal. Acc. $\uparrow$) & 0.861 (0.011) & \textbf{0.925 (0.003)} & 0.884 (0.039) & 0.847 (0.006) & 0.869 (0.025) & 0.728 (0.030)\\
SelfRegulationSCP1 (Bal. Acc. $\uparrow$) & 0.731 (0.018) & 0.774 (0.010) & 0.770 (0.002) & 0.789 (0.045) & \underline{0.801 (0.019)} & \textbf{0.815 (0.054)}\\
SelfRegulationSCP2 (Bal. Acc. $\uparrow$) & 0.494 (0.024) & 0.502 (0.032) & 0.509 (0.033) & 0.483 (0.035) & \textbf{0.530 (0.062)} & 0.494 (0.048)\\
SpokenArabicDigits (Bal. Acc. $\uparrow$) & 0.987 (0.001) & 0.987 (0.003) & \underline{0.988 (0.002)} & 0.975 (0.004) & 0.979 (0.003) & 0.979 (0.007)\\
StandWalkJump (Bal. Acc. $\uparrow$) & \underline{0.333 (0.000)} & 0.289 (0.038) & \textbf{0.422 (0.139)} & 0.311 (0.038) & \underline{0.333 (0.067)} & 0.289 (0.102)\\
UWaveGestureLibrary (Bal. Acc. $\uparrow$) & 0.848 (0.015) & \textbf{0.884 (0.011)} & 0.653 (0.151) & 0.820 (0.024) & 0.818 (0.008) & 0.728 (0.019)\\
MFPT Bearing (Bal. Acc. $\uparrow$) & 0.733 (0.149) & \underline{0.953 (0.095)} & 0.875 (0.086) & 0.729 (0.090) & 0.704 (0.246) & 0.671 (0.020)\\
Paderborn (KAT) (Bal. Acc. $\uparrow$) & \textbf{0.612 (0.021)} & 0.369 (0.080) & 0.507 (0.020) & 0.316 (0.088) & 0.304 (0.125) & 0.565 (0.112)\\
\bottomrule
\end{tabular}
\caption{Panel (c), continued: UEA multivariate classification and vibration/fault diagnosis, comparison block 2. ALPHABET is repeated to keep both comparison blocks directly readable.}
\label{tab:posthoc-pointwise-uea-fault-continued}
\end{table*}

\subsection{Configuration freezing and final evaluation}

The control-family candidates are listed in Table~\ref{tab:model-config-ledger}, and the proposed \model{} capacity schedule is specified in the Evaluation Protocol. Each task--family cell contains 18 candidate configurations. In Stage~1, all 18 candidates are evaluated at seed 7, and the six configurations with the highest validation scores are retained. In Stage~2, these six configurations are evaluated at seeds 11 and 19, and the configuration with the highest mean validation score across seeds \(\{7,11,19\}\) is frozen for final training. Exact ties are resolved using the predeclared configuration key.

Regular-sequence tasks are trained for at most 100 epochs, with validation checkpointing and an early-stopping patience of eight epochs.

For each task--family cell, the selection and final-evaluation protocol requires 18 Stage~1 fits, 12 additional Stage~2 fits, and five final-training fits. Across 82 tasks and ten model families, this yields a total of 28,700 fits.

Within each benchmark suite, all model families use the same epoch limit, validation split, checkpointing rule, and selection metric. Neither configuration-selection stage accesses the TEST data, and each campaign manifest fixes the final-training seed set before TEST evaluation.

For UCR datasets, the selected configuration is retrained on the full official TRAIN set for the rounded median of its validation-selected epoch counts. For the remaining datasets, final evaluation uses the predefined TRAIN split and restores the validation-loss checkpoint obtained under the frozen configuration. Only the configuration selected by the two-stage validation procedure contributes to the reported TEST result; results from nonselected configurations do not affect TEST reporting or cross-model ranking.

\subsection{Final metrics and PhysioNet handling}

The metric contract is fixed per dataset in the registry. UCR, UEA,
fault-diagnosis, Human Activity, ISRUC-Sleep, and Chapman-Shaoxing rows report
balanced accuracy; forecasting rows report MSE. AudioSet reports macro AUPRC.
The remaining standard multiclass rows report accuracy, and the PTB-XL and
PhysioNet rows report macro AUROC or AUROC, respectively. Every row in
Table~\ref{tab:posthoc-pointwise-results} names its metric and direction.
  
Where selection and reporting differ, validation selection uses the
imbalance-sensitive metric fixed before TEST access---balanced accuracy for
standard multiclass tasks, macro AUPRC/AUPRC for multilabel or binary clinical
tasks, and macro-F1 for multiclass PAMAP2---while the primary row retains the
dataset's conventional accuracy or AUROC endpoint for comparison with prior
work. The indented companion rows in
Table~\ref{tab:posthoc-pointwise-results} report the corresponding final-TEST
selection metric for the same frozen configurations and seeds; these rows are
not used for reselection or cross-task ranking. For \(K\) classes, balanced
accuracy is \(K^{-1}\sum_{k=1}^K\mathrm{TP}_k/
(\mathrm{TP}_k+\mathrm{FN}_k)\). Macro AUROC/AUPRC average one-vs-rest label
scores without prevalence weighting, and MSE averages squared error over the
reported forecast targets.

\subsection{Dataset-specific protocol: PhysioNet 2012}
\label{app:causal-clinical-results}

The PhysioNet 2012 endpoint uses ICU mortality records
\cite{silva2012physionet,goldberger2000physionet}. Original minute timestamp
groups are retained without binning or forward filling. Inputs combine
standardized values, observation indicators, elapsed time, and timestep
validity; \model{} receives the last two as metadata and controls as channels.

Set A is split into 3,200 optimization and 800 validation records; Set C remains
sealed until configuration freezing, and standardization uses optimization
records only. The preselection-frozen Set-C parser ignores nameless
measurements, rejects non-empty unknown variables, and right-truncates beyond
the Set-A-selected 208-step geometry (Set-A maximum: 203 groups). AUROC and
AUPRC are threshold-free; balanced accuracy uses the Set-A-validation threshold.

\subsection{Completeness, ranking, and statistical scope}

Aggregation requires every key in each campaign's expected-key manifest.
Per-task ranks follow the stated metric direction; values satisfying
\texttt{isclose} with \(\mathrm{rtol}=10^{-5}\), \(\mathrm{atol}=10^{-8}\)
share average rank. Registry summaries average ranks, not heterogeneous raw
metrics. Friedman tests use tasks as units; pairwise signed-rank tests operate
on taskwise rank differences, retain zeros by Pratt's convention, and use Holm
correction \cite{demsar2006statistical}.